\documentclass[jair,twoside,11pt,theapa]{article}
\usepackage{jair,theapa,rawfonts}

\usepackage{subcaption}
\usepackage{url}
\usepackage{algorithm}
\usepackage[noend]{algorithmic}
\usepackage{tikz}
\usepackage{mathtools}
\usepackage{amssymb}
\usepackage{amsmath}
\usepackage{xspace}

\newcommand{\expect}{\mathbb{E}}

\newcommand{\reals}{\mathbb{R}}
\newcommand{\tuple}[1]{\langle{#1}\rangle}

\DeclareMathOperator*{\argmax}{arg\,max}

\newcommand{\mdp}{\tuple{S,A,r,p,\gamma}}

\newcommand{\qapprox}{\tilde{q}}

\newcommand{\decoration}{\ding{91}\xspace}

\newcommand{\true}{\ensuremath{\mathsf{true}}\xspace{}}
\newcommand{\false}{\ensuremath{\mathsf{false}}\xspace{}}

\makeatletter
\DeclareRobustCommand\onedot{\futurelet\@let@token\@onedot}
\def\@onedot{\ifx\@let@token.\else.\null\fi\xspace}

\def\ie{{i.e}\onedot}

\makeatother

\usepackage{tikz}
\usepackage{pgfplots,pgfplotstable}
\usepgfplotslibrary{fillbetween}
\pgfplotsset{compat=1.10}

\newcommand{\percentiles}[7]{
\pgfplotstableread{#1}\datatable
  \addplot [name path=pluserror,draw=none,no markers,forget plot]
    table [x={#2},y expr=\thisrow{#5}] {\datatable};

  \addplot [name path=minuserror,draw=none,no markers,forget plot]
    table [x={#2},y expr=\thisrow{#3}] {\datatable};

  \addplot [forget plot,fill=#6,opacity=#7]
    fill between[on layer={},of=pluserror and minuserror];

  \addplot [#6,thick,no markers,line width=1.3pt]
    table [x={#2},y={#4}] {\datatable};
}

\newcommand{\entry}[2]
{
    \raisebox{3pt}{\tikz{\draw[#1,line width=1.3pt] (0,0) -- (0.6,0);}} #2
}

\newcommand{\entrysmall}[2]
{
    \raisebox{3pt}{\tikz{\draw[#1,line width=1.5pt] (0,0) -- (0.5,0);}} #2
}

\usepackage{amsthm}
\newtheorem{theorem}{Theorem}[section]

\newtheorem{observation}{Observation}
\newtheorem{definition}[theorem]{Definition}
\usepackage{tikz}
\usetikzlibrary{shapes.geometric}
\usetikzlibrary{positioning,automata,arrows}
\usepackage{pifont}
\usepackage{fontawesome,wasysym,marvosym}

\newcommand{\coffee}[0]{{\color{black}\Coffeecup}\xspace}
\newcommand{\mail}[0]{\Letter}

\newcommand{\agent}{\resizebox{4mm}{!}{\begin{tikzpicture}\node[draw, thick, shape border rotate=90, isosceles triangle, isosceles triangle apex angle=60, fill=violet!70!white, fill opacity=1.0, node distance=1cm,minimum height=1.5em] at (0,0) {};\end{tikzpicture}}\xspace}
\usepackage{booktabs}
\usepackage{tabularx}
\usepackage{multirow}
\jairheading{73}{2022}{173-208}{10/2020}{01/2022}
\ShortHeadings{Reward Machines: Exploiting Reward Function Structure in RL}
{Toro Icarte, Klassen, Valenzano, \& McIlraith}
\firstpageno{173}

\begin{document}

\title{Reward Machines: Exploiting Reward Function\\Structure in Reinforcement Learning}

\author{\name Rodrigo Toro Icarte \email rntoro@uc.cl \\
       \addr Pontificia Universidad Católica de Chile, Santiago, Chile\\
       Vector Institute, Toronto, ON, 
       Canada
       \AND
       \name Toryn Q. Klassen \email toryn@cs.toronto.edu \\
       \addr University of Toronto, Toronto, ON, 
       Canada
       \\
       Vector Institute, Toronto, ON, 
       Canada
       \AND
       \name Richard Valenzano \email rick.valenzano@ryerson.ca \\
       \addr Ryerson University, Toronto, ON, 
       Canada
       \AND
       \name Sheila A. McIlraith \email sheila@cs.toronto.edu \\
       \addr University of Toronto, Toronto, ON, 
       Canada
       \\
       Vector Institute, Toronto, ON,  
       Canada
       }


\maketitle

\begin{abstract}
Reinforcement learning (RL) methods usually treat reward functions as black boxes. As such, these methods must extensively interact with the environment in order to discover rewards and optimal policies. In most RL applications,  however,  users have to program the reward function and, hence, there is the opportunity to make the reward function visible  –  to  show  the  reward  function’s  code  to  the  RL  agent  so  it  can  exploit the function’s  internal  structure  to  learn  optimal  policies  in a more sample efficient manner. In  this  paper,  we  show  how  to accomplish this idea in two steps. First, we propose reward machines, a type of finite state machine that supports the specification of reward functions while exposing reward function  structure.   We  then  describe  different  methodologies  to  exploit  this structure to support learning, including automated reward shaping, task decomposition, and counterfactual reasoning with off-policy learning.  Experiments on tabular and continuous domains, across different tasks and RL agents,  show the benefits of exploiting reward structure with respect to sample efficiency and the quality of resultant policies.  Finally, by virtue of being a form of finite state machine, reward machines have the expressive power of a regular language and as such support loops, sequences and conditionals, as well as the expression of temporally extended properties typical of linear temporal logic and non-Markovian reward specification. 
\end{abstract}

\section{Introduction}

A standard assumption in reinforcement learning (RL) is that the agent does not have access to the environment model \shortcite{sutton1998reinforcement}. This means that it does not know the environment's transition probabilities or reward function. To learn optimal behaviour, an RL agent must therefore interact with the environment and learn from its experience. While assuming that the transition probabilities are unknown seems reasonable, there is less reason to hide the reward function from the agent. Artificial agents cannot inherently perceive reward from the environment; someone must program those reward functions. This is true even if the agent is interacting with the real world. Typically, though, a programmed reward function is given as a black box to the agent. The agent can query the function for the reward in the current situation, but does not have access to whatever structures or high-level ideas the programmer may have used in defining it. However, an agent that had access to the specification of the reward function might be able to use such information to learn optimal policies faster. We consider different ways to do so in this work.

Previous work on giving an agent knowledge about the reward function focused on defining a task specification language -- usually based on sub-goal sequences \shortcite{singh1992reinforcement,singh1992transfer} or linear temporal logic \shortcite{li2016reinforcement,littman2017environment,aamas2018lpopl,hasanbeig2018logically,camamacho2019ijcai,de2019foundations,shah2020planning}  -- and then generate a reward function towards fulfilling that specification. In this work, we instead directly tackle the problem of defining reward functions that expose structure to the agent. As such, our approach is able to reward behaviours to varying degrees in manners that cannot be expressed by previous approaches.

There are two main contributions of this work. First, we introduce a type of finite state machine, called a \emph{reward machine}, which we use in defining rewards. A reward machine allows for composing different reward functions in flexible ways, including concatenation, 
loops, and conditional rules. As an agent acts in the environment, moving from state to state, it also moves from state to state within a reward machine (as determined by high-level events detected within the environment). After every transition, the reward machine outputs the reward function the agent should use at that time. For example, we might construct a reward machine for ``\emph{delivering coffee to an office}" using two states. In the first state, the agent does not receive any rewards, but it moves to the second state whenever it gets the coffee. In the second state, the agent gets rewards after delivering the coffee. The advantage of defining rewards this way is that the agent knows that the problem consists of two stages and might use this information to speed up learning.

Our second contribution is a collection of RL methods that can exploit a reward machine's internal structure to improve sample efficiency. These methods include using the reward machine for decomposing the problem, shaping the reward functions, and using counterfactual reasoning in conjunction with off-policy learning to learn policies in a more sample efficient manner. We also discuss conditions under which these approaches are guaranteed to converge to optimal policies and empirically demonstrate the value of exploiting reward structures in discrete and continuous domains.

Note that reward functions need not be specified as a reward machine natively to benefit from the learning methods presented in this paper. Rather, functions can be specified in a diversity of languages and automatically translated to a reward machine as argued by \shortciteA{camamacho2019ijcai} and realized to a degree by \shortciteA{middleton2020icaps}. Further, reward machines can be learned from data and from demonstrations as we discuss in Section~\ref{sec:related_work}.

This paper builds upon our previous work \shortcite{icml2018rms} -- where we originally proposed reward machines and an approach, called \emph{Q-learning for reward machines (QRM)}, to exploit the structure exposed by a reward machine. This paper also covers an approach for automated reward shaping from a given reward machine that we later introduced \shortcite{camamacho2019ijcai}. Since then, we have gathered additional practical experience and theoretical understanding about reward machines that are reflected in this paper. Concretely, we provide a cleaner definition of a reward machine and propose two novel approaches to exploit its structure, called \emph{counterfactual experiences for reward machines (CRM)} and \emph{hierarchical RL for reward machines (HRM)}. We expanded the related work discussion to include recent trends in reward machine research and included new empirical results in single task, multitask, and continuous control learning problems. Finally, we have released a new implementation of our code that is fully compatible with the OpenAI Gym API \shortcite{1606.01540}. We hope that this paper and code will facilitate future research on reward machines.

\section{Reinforcement Learning}
\label{sec:preliminaries}

The RL problem consists of an agent interacting with an unknown environment. Usually, the environment is modeled as a \emph{Markov decision process (MDP)}. An MDP is a tuple $\mathcal{M}=\mdp$ where $S$ is a finite set of \emph{states}, $A$ is a finite set of \emph{actions}, $r:S \times A\times S \rightarrow \reals$ is the \emph{reward function}, $p(s_{t+1}|s_t,a_t)$ is the \emph{transition probability distribution}, and $\gamma\in(0,1]$ is the \emph{discount factor}. In some cases, a subset of the states are labelled as \emph{terminal states}.

A \emph{policy} $\pi(a|s)$ is a probability distribution over the actions $a \in A$ given a state $s \in S$. At each time step $t$, the agent is in a particular state $s_t \in S$, selects an action $a_t$ according to $\pi(\cdot|s_t)$, and executes $a_t$. The agent then receives a new state $s_{t+1} \sim p(\cdot|s_t,a_t)$ and an immediate reward $r(s_t,a_t,s_{t+1})$ from the environment. The process then repeats from $s_{t+1}$ until potentially reaching a terminal state. The agent's goal is to find an \emph{optimal policy} $\pi^*$ that maximizes the expected discounted return $G_t = \expect_{\pi}{\left[\sum_{k=0}^{\infty}\gamma^kr_{t+k} \middle| S_t=s\right]}$ when starting from any state $s \in S$ and time step $t$.

The \emph{Q-function} $q^{\pi}(s,a)$ under a policy $\pi$ is defined as the expected discounted return of taking action $a$ in state $s$ and then following policy $\pi$. It is known that every optimal policy $\pi^*$ satisfies the \emph{Bellman optimality} equations (where $q^*=q^{\pi^*}$): 
\begin{align}
q^*(s,a) =  \sum_{s'\in S}{p(s'|s,a) \left( r(s,a,s') + \gamma \max_{a' \in A} q^*(s',a') \right)} \;\;, 
\end{align}
for every state $s \in S$ and action $a \in A$. Note that, if $q^{*}$ is known, then an optimal policy can be computed by always selecting the action $a$ with the highest value of $q^*(s,a)$. 

\subsection{Tabular Q-Learning}

Tabular Q-learning \cite{watkins1992q} is a well-known approach for RL. This algorithm works by using the agent's experience to estimate the optimal Q-function. We denote this Q-value estimate as $\qapprox(s,a)$. On every iteration, the agent observes the current state $s$ and chooses an action $a$ according to some exploratory policy. One common exploratory policy is the $\epsilon$-greedy policy, which selects a random action with probability $\epsilon$, and $\argmax_a{\qapprox(s,a)}$ with probability $1-\epsilon$. Given the resulting state $s'$ and immediate reward $r(s,a,s')$, this experience is used to update $\qapprox(s,a)$ as follows:
\begin{align}
\qapprox(s,a) \xleftarrow{\alpha} r(s,a,s') + \gamma \max_{a'}{\qapprox(s',a')} \;\;,
\end{align}
where $\alpha$ is an hyperparameter called the \emph{learning rate}, and we use $x \xleftarrow{\alpha} y$ as shorthand notation for $x \leftarrow x + \alpha \cdot (y - x)$. Note that $\qapprox(s,a) \xleftarrow{\alpha} r(s,a,s')$ when $s'$ is a terminal state.

Tabular Q-learning is guaranteed to converge to an optimal policy in the limit as long as each state-action pair is visited infinitely often. This algorithm is an \emph{off-policy} learning method since it can learn from the experience generated by any policy. Unfortunately, tabular Q-learning is impractical when solving problems with large state spaces. In such cases, \emph{function approximation} methods like DQN are often used. 

\subsection{Deep Q-Networks (DQN)}

\emph{Deep Q-Network (DQN)}, proposed by \shortciteA{mnih2015human}, is a method which approximates the Q-function with an estimate $\qapprox_\theta(s,a)$ using a deep neural network with parameters $\theta$.
To train the network, mini-batches of experiences $(s,a,r,s')$ are randomly sampled from an \emph{experience replay} buffer and used to minimize the square error between $\qapprox_\theta(s,a)$ and the Bellman estimate $r + \gamma \max_{a'}{\qapprox_{\theta'}(s',a')}$. The updates are made with respect to a \emph{target network} with parameters $\theta'$. The parameters $\theta'$ are held fixed when minimizing the square error, but updated to $\theta$ after a certain number of training updates. The role of the target network is to stabilize learning. DQN inherits the off-policy behaviour from tabular Q-learning, but is no longer guaranteed to converge to an optimal policy.

Since its original publication, several improvements have been proposed to DQN. We consider one of them in this paper: \emph{Double DQN} \shortcite{van2016deep}. Double DQN uses two neural networks, parameterized by $\theta$ and $\theta'$, to decouple action selection from value estimation, and thereby decrease the overestimation bias that DQN is known to suffer from. As such, double DQN usually outperforms DQN while preserving its off-policy nature.

\subsection{Deep Deterministic Policy Gradient (DDPG)}

DQN cannot solve continuous control problems because DQN's network has one output unit per possible action and the space of possible actions is infinite in continuous control problems. For those cases, actor-critic approaches such as \emph{Deep Deterministic Policy Gradient (DDPG)} \shortcite{lillicrap2015continuous} are preferred. DDPG is an off-policy actor-critic approach that also uses neural networks for approximating the Q-value $\qapprox_\theta(s,a)$. However, in this case 
the action can take continuous values. To decide which action to take in a given state $s$, an actor network $\pi_\mu(s)$ is learned. The actor network receives the current state and outputs a (possibly continuous) action to execute in the environment. 

Training $\qapprox_\theta(s,a)$ is done by minimizing the Bellman error and letting the actor policy select the next action. Given a set of experiences $(s,a,r,s')$ sampled from the experience replay buffer, $\theta$ is updated towards minimizing the square error between $\qapprox_\theta(s,a)$ and $r + \gamma \qapprox_{\theta'}(s',\pi_{\mu'}(s'))$, where $\theta'$ and $\mu'$ are the parameters of target networks for the Q-value estimate and the actor policy. Training the actor policy $\pi_\mu(s)$ is done by moving its output towards $\argmax_a \qapprox_\theta(s,a)$. To do so, $\pi_\mu(s)$ is updated using the expected gradient of $\qapprox_\theta(s,a)$ when $a=\pi_\mu(s)$: $\nabla_{\mu}[\qapprox_\theta(s,a)|s=s_t,a=\pi_\mu(s_t)]$. This gradient is approximated using sampled mini-batches from the experience replay buffer. Finally, the target network's parameters are periodically updated as follows: $\theta' \xleftarrow{\tau} \theta$ and $\mu' \xleftarrow{\tau} \mu$, where $\tau \in (0,1)$.

\section{Reward Machines}
\label{sec:reward-machines}

In this section, we introduce a novel type of finite state machine, called a \emph{reward machine (RM)}. An RM takes abstracted descriptions of the environment as input, and outputs reward functions. The intuition is that the agent will be rewarded by different reward functions at different times, depending on the state in the RM. Hence, an RM can be used to define temporally extended (and as such, non-Markovian) tasks and behaviours. We then show that an RM can be interpreted as specifying a single reward function over a larger state space, and consider types of reward functions that can be expressed using RMs.

As a running example, consider the \emph{office gridworld} presented in Figure~\ref{fig:officeleft}. In this environment, the agent can move in the four cardinal directions. It picks up coffee if at location \coffee, picks up the mail if at location \mail, and delivers the coffee and mail to an office if at location $o$. The building contains decorations \decoration, which the agent breaks if it steps on them. Finally, there are four marked locations: $A$, $B$, $C$, and $D$. In the rest of this section, we will show how to define tasks for an RL agent in this environment using RMs.

A reward machine is defined over a set of propositional symbols $\mathcal{P}$. Intuitively, $\mathcal{P}$ is a set of relevant high-level events from the environment that the agent can detect. In the office gridworld, we can define $\mathcal{P}=\{\text{\coffee},\text{\mail},o,\text{\decoration},A,B,C,D\}$, where event $e\in\mathcal{P}$ occurs when the agent is at location $e$. We can now formally define a reward machine as follows:
%
\begin{definition}[reward machine]Given a set of propositional symbols $\mathcal{P}$, a set of (environment) states $S$, and a set of actions $A$, a reward machine (RM) is a tuple $\mathcal{R}_{\mathcal{P}SA}=\tuple{U,u_0, F,\delta_u,\delta_r}$ where $U$ is a finite set of states, $u_0 \in U$ is an initial state, $F$ is a finite set of terminal states (where $U \cap F = \emptyset$), $\delta_u$ is the state-transition function, $\delta_u: U \times 2^\mathcal{P} \rightarrow U \cup F$, and $\delta_r$ is the state-reward function, $\delta_r: U \rightarrow [S \times A\times S \rightarrow \reals]$.
\end{definition}

A reward machine $\mathcal{R}_{\mathcal{P}SA}$ starts in state $u_0$, and at each subsequent time is in some state $u_t \in U \cup F$. At every step $t$, the machine receives as input a \emph{truth assignment} $\sigma_t$, which is a set that contains exactly those propositions in $\mathcal{P}$ that are currently true in the environment. For example, in the office gridworld, $\sigma_t=\{e\}$ if the agent is at a location marked as $e$. Then the machine moves to the next state $u_{t+1}=\delta_u(u_t,\sigma_t)$ according to the state-transition function, and outputs a reward function $r_t=\delta_r(u_t)$ according to the state-reward function. This process repeats until the machine reaches a terminal state. Note that reward machines can model never-ending tasks by defining $F=\emptyset$.

\begin{figure}
    \centering
    \begin{subfigure}[t]{.5\columnwidth}
    \centering
    \begin{tikzpicture}[scale=0.54]
    \draw[step=1cm,gray] (0,0) grid (12, 9);
    \draw[ultra thick] (3,0) -- (3,1);
    \draw[ultra thick] (3,2) -- (3,7);
    \draw[ultra thick] (3,8) -- (3,9);
    \draw[ultra thick] (6,0) -- (6,1);
    \draw[ultra thick] (6,2) -- (6,7);
    \draw[ultra thick] (6,8) -- (6,9);
    \draw[ultra thick] (9,0) -- (9,1);
    \draw[ultra thick] (9,2) -- (9,7);
    \draw[ultra thick] (9,8) -- (9,9);
    
    \draw[ultra thick] (0,3) -- (1,3);
    \draw[ultra thick] (2,3) -- (10,3);
    \draw[ultra thick] (11,3) -- (12,3);
    \draw[ultra thick] (0,6) -- (1,6);
    \draw[ultra thick] (2,6) -- (4,6);
    \draw[ultra thick] (5,6) -- (7,6);
    \draw[ultra thick] (8,6) -- (10,6);
    \draw[ultra thick] (11,6) -- (12,6);
    \node at (1.5,1.5) {A};
    \node at (10.5,1.5) {D};
    \node at (10.5,7.5) {C};
    \node at (1.5,7.5) {B};
    \node at (3.5,6.5) {\coffee};
    \node at (8.5,2.5) {\coffee};
    \node at (7.5,4.5) {\Letter};
    \node at (1.5,4.5) {\decoration};
    \node at (4.5,1.5) {\decoration};
    \node at (7.5,1.5) {\decoration};
    \node at (4.5,7.5) {\decoration};
    \node at (7.5,7.5) {\decoration};
    \node at (10.5,4.5) {\decoration};

    \node at (4.5,4.5) {o};
    \node at (2.5,1.5) {\agent};
    \draw[ultra thick, ->, >=stealth, draw=red!70!white] (2.8,1.5) -- (3.5,1.5) -- (3.5,2.5) -- (5.5,2.5) -- (5.5,1.5) -- (6.5,1.5) -- (6.5,2.5) -- (8.2,2.5);
    \draw[ultra thick, ->, >=stealth, draw=red!70!white] (8.5,2.2) -- (8.5,1.5) -- (9.5,1.5) -- (9.5,2.5) -- (10.5,2.5) -- (10.5,3.5) -- (9.5,3.5) -- (9.5,5.5) -- (10.5,5.5) -- (10.5,6.5) -- (9.5,6.5) -- (9.5,7.5) -- (8.5,7.5) -- (8.5,6.5) -- (6.5,6.5) -- (6.5,7.5) -- (5.5,7.5)-- (5.5,6.5) -- (4.7,6.5) -- (4.7,4.8);
    \draw[ultra thick, ->, >=stealth, draw=blue!70!white] (2.5,1.9) -- (2.5,2.5) -- (1.5,2.5) -- (1.5,3.5) -- (2.5,3.5) -- (2.5,5.5) -- (1.5,5.5) -- (1.5,6.5) -- (2.5,6.5) -- (2.5,7.5) -- (3.5,7.5) -- (3.5,6.8);
    \draw[ultra thick, ->, >=stealth, draw=blue!70!white] (3.8,6.5) -- (4.3,6.5) -- (4.3,4.8);
    
    \draw[ultra thick] (0,0) rectangle (12,9);
\end{tikzpicture}
    \subcaption{The office gridworld
    }
    \label{fig:officeleft}
    \end{subfigure}
    \begin{subfigure}[t]{.47\columnwidth}
    \centering
    \begin{tikzpicture}[node distance=2cm,on grid,every initial by arrow/.style={ultra thick,->, >=stealth}]
  \node[ultra thick,state,initial above] (q_0) at (0,0) {$u_0$};
  \node[ultra thick,state]         (q_1) at (0,-1.8)  {$u_1$};
  \node[circle,draw=black,minimum size=0.26cm,inner sep=0pt,fill=black] (t1) at (2,0)  {};
  \node[circle,draw=black,minimum size=0.26cm,inner sep=0pt,fill=black] (t2) at (2,-1.8)  {};
  \node[circle,draw=black,minimum size=0.26cm,inner sep=0pt,fill=black] (t3) at (0,-3.3)  {};
  \node[text width=1cm] at (.75,-3.3) {$t$};
  
  \path[ultra thick,->, >=stealth] (q_0) edge node [left] {$\tuple{\text{\coffee} \wedge \neg \text{\decoration},0}$} (q_1);
  \path[ultra thick,->, >=stealth] (q_0) edge [loop left] node {$\tuple{\neg \text{\coffee} \wedge \neg \text{\decoration},0}$} ();
  \path[ultra thick,->, >=stealth] (q_0) edge node [above] {$\tuple{\text{\decoration},0}$} (t1);
  \path[ultra thick,->, >=stealth] (q_1) edge [loop left] node {$\tuple{\neg o \wedge \neg \text{\decoration},0}$} ();
  \path[ultra thick,->, >=stealth] (q_1) edge node [above] {$\tuple{\text{\decoration},0}$} (t2);
  \path[ultra thick,->, >=stealth] (q_1) edge node [left]{$\tuple{o \wedge \neg \text{\decoration},1}$} (t3);
\end{tikzpicture}
    \subcaption{A simple reward machine}
    \label{fig:officeright}
    \end{subfigure}
    \caption{An example environment and one reward machine for it}
    \label{fig:office}
\end{figure}
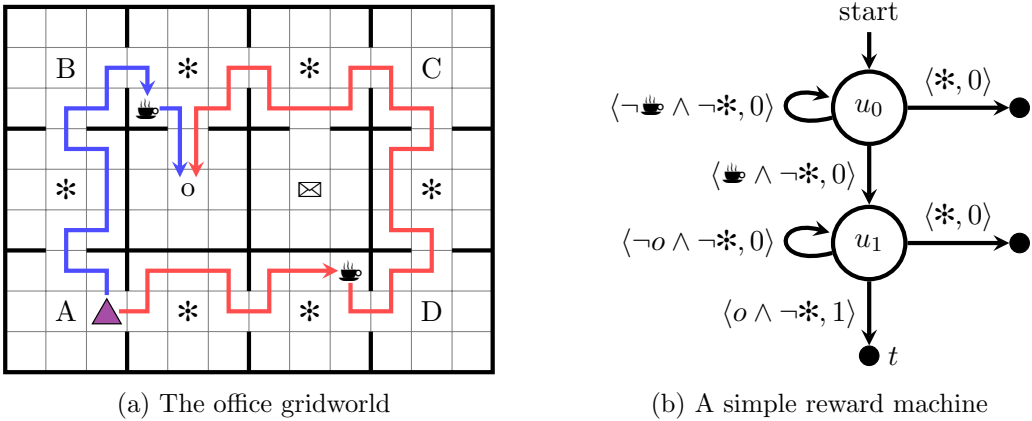

In our examples, we will be considering \emph{simple} reward machines (which we later prove are a particular case of reward machines), defined as follows:
\begin{definition}[simple reward machine]
Given a set of propositional symbols $\mathcal{P}$, a simple reward machine is a tuple $\mathcal{R}_{\mathcal{P}}=\tuple{U,u_0, F,\delta_u,\delta_r}$ where $U$, $u_0$, $F$, and $\delta_u$ are defined as in a standard reward machine,
but the state-reward function  $\delta_r: U \times 2^\mathcal{P} \rightarrow \reals$ depends on $2^\mathcal{P}$ 
and returns a number instead of a function.
\end{definition}

Figure~\ref{fig:officeright} shows a graphical representation of a simple reward machine for the office gridworld. Every node in the graph is a state of the machine, $u_0$ being the initial state. Terminal states are represented by black circles. Each edge is labelled by a tuple $\tuple{\varphi,c}$, where $\varphi$ is a propositional logic formula over $\mathcal{P}$ and $c$ is a real number. An edge between $u_i$ and $u_j$ labelled by $\tuple{\varphi,c}$ means that $\delta_u(u_i,\sigma)=u_j$ whenever $\sigma \models \varphi$ (\ie, the truth assignment $\sigma$ satisfies $\varphi$), and $\delta_r(u_i,\sigma)$ returns a reward of $c$. For instance, the edge between the state $u_1$ and the terminal state $t$ labelled by $\tuple{o \wedge \neg \text{\decoration},1}$ means that the machine will transition from $u_1$ to $t$ if the proposition $o$ becomes $\true$ and \decoration is $\false$, and output a reward of one. Intuitively, this machine outputs a reward of one if and only if the agent delivers coffee to the office while not breaking any of the decorations. The blue path in Figure~\ref{fig:officeleft} shows an optimal way to complete this task, and the red path shows a sub-optimal way.

Now that we have defined a reward machine, we can use it to reward an agent. 
To do so, we require a \emph{labelling function} $L:S \times A \times S \to 2^\mathcal{P}$. $L$ assigns truth values to symbols in $\mathcal{P}$ given an environment experience $e = (s,a,s')$, where $s'$ is the resulting state after executing action $a$ from state $s$. 
The labelling function plays the key role of producing the truth assignments that are input to the reward machine, as discussed below.

\begin{definition}
A Markov decision process with a reward machine (\nobreak{MDPRM}) is a tuple $\mathcal{T}=\tuple{S,A,p,\gamma,\mathcal{P},L,U,u_0,F,\delta_u,\delta_r}$, where $S,A,p,$ and $\gamma$ are defined as in an MDP, $\mathcal{P}$ is a set of propositional symbols, $L$ is a labelling function $L:S\times A \times S \to 2^\mathcal{P}$, and $U,u_0,F,\delta_u,$ and $\delta_r$ are defined as in a reward machine.
\end{definition}
The RM in an MDPRM $\mathcal{T}$ is updated at every step of the agent in the environment. If the RM is in state $u$ and the agent performs action $a$ to move from state $s$ to $s'$ in the MDP, then the RM moves to state $u'=\delta_u(u,L(s,a,s'))$ and the agent receives a reward of $r(s,a,s')$, where $r=\delta_r(u)$. For a simple reward machine, the reward is $\delta_r(u,L(s,a,s'))$.

In the running example (Figure~\ref{fig:office}), for instance, the reward machine starts in $u_0$ and stays there until the agent reaches a location marked with \decoration or \coffee. If \decoration is reached (\ie, a decoration is broken), the machine moves to a terminal state, ending the episode and providing no reward to the agent. In contrast, if \coffee is reached, the machine moves to $u_1$. While the machine is in $u_1$, two outcomes might occur. The agent might reach a \decoration, moving the machine to a terminal state and returning no reward, or it might reach the office $o$, also moving the machine to a terminal state but giving the agent a reward of 1. 

Note that the rewards the agent gets may be non-Markovian relative to the environment (the states of $S$), though they are Markovian relative to the elements in $S\times U$. As such, when making decisions on what action to take in an MDPRM, the agent should consider not just the current environment state $s_t\in S$ but also the current RM state $u_t\in U$. 

A policy $\pi(a|\tuple{s,u})$ for an MDPRM is a probability distribution over actions $a\in A$ given a pair $\tuple{s,u}\in S\times U$.
We can think of an MDPRM as defining an MDP with state set $S\times U$, as described in the following observation.

\begin{observation}
    \label{obs:cross-mdp}
    Given an MDPRM $\mathcal{T}=\tuple{S, A, p, \gamma, \mathcal{P}, L, U, u_0, F, \delta_u, \delta_r}$, let $\mathcal{M}_\mathcal{T}$ 
    be the MDP $\tuple{S',A',r',p',\gamma'}$ defined such that $S' = S \times (U\cup F)$, $A' = A$, $\gamma'=\gamma$, 
    \begin{align*}
    p'(\tuple{s',u'}|\tuple{s,u},a)=\begin{cases} p(s'|s,a) &\text{ if } u \in F \text{ and } u' = u\\p(s'|s,a) &\text{ if } u \in U \text{ and } u' = \delta_u(u,L(s,a,s'))\\0&\text{ otherwise}\end{cases}
    \end{align*}
    and $r'(\tuple{s,u},a,\tuple{s',u'}) = \delta_r(u)(s,a,s')$ if $u \not\in F$ (zero otherwise). Then any policy for $\mathcal{M}_\mathcal{T}$ achieves the same expected discounted return in $\mathcal{T}$, and vice versa.
\end{observation}

We can now see why simple reward machines are a particular case of reward machines. Basically, for any labelling function, we can set the Markovian reward functions in a reward machine to mimic the reward given by a simple reward machine, as shown below.
\begin{observation}
    \label{prop:srm-are-rm}
    Given any labelling function $L:S \times A \times S \to 2^\mathcal{P}$, a simple reward machine $\mathcal{R}_{\mathcal{P}}=\tuple{U,u_0, F,\delta_u,\delta_r}$ is equivalent to a reward machine $\mathcal{R}_{\mathcal{P}SA}=\tuple{U,u_0, F,\delta_u,\delta_r'}$ where $\delta_r'(u)(s,a,s') = \delta_r(u,L(s,a,s'))$ for all $u \in U$, $s \in S$, $a \in A$, and $s' \in S$. That is, both $\mathcal{R}_{\mathcal{P}}$ and $\mathcal{R}_{\mathcal{P}SA}$ will be at the same RM state and output the same reward for every possible sequence of environment state-action pairs.
\end{observation}

Finally, we note that RMs can express any Markovian and some non-Markovian reward functions. In particular, given a set of states $S$ and actions $A$, the following properties hold: 
\begin{enumerate}
    \item Any Markovian reward function $R:S\times A \times S \rightarrow \reals$ can be expressed by a reward machine with one state.
    \item A non-Markovian reward function $R:(S \times A)^* \rightarrow \reals$  can be expressed using a reward machine if the reward depends on the state and action history $(S \times A)^*$ only to the extent of distinguishing among those histories that are described by different elements of a finite set of regular expressions over elements in $S \times A \times S$.
    \item Non-Markovian reward functions $R:(S \times A)^* \rightarrow \reals$ that distinguish between histories via properties not expressible as regular expressions over elements in $S \times A \times S$ (such as counting how many times a state has been reached) cannot be expressed using a reward machine.
\end{enumerate}

In other words, reward machines can return different rewards for the same transition $(s,a,s')$ in the environment, for different histories of states and actions seen by the agent, as long as the history can be represented by a regular language. This holds because regular languages are exactly those that are accepted by deterministic finite state automata \cite{hopcroftautomatatheory}. As such, RMs can specify structure in the reward function that includes loops, conditional statements, and sequence interleaving, as well as behavioral constraints, such as safety constraints. For example, the first task used in Section \ref{sec:res-cheetah} has loops, task 4 in Table \ref{tab:wwt} involves sequence interleaving, and task 1 in Table \ref{tab:owt} includes a safety constraint. To allow for structure beyond what is expressible by regular languages requires that the agent has access to an external memory, which we leave as future work.

\paragraph{Relationship to Mealy and Moore Machines}
A reader familiar with automata theory will recognize that, except for the terminal states, reward machines are Moore machines (with an output alphabet of reward functions) and simple reward machines are Mealy machines (with an output alphabet of numbers). As such, it seems reasonable to consider a more general form of Mealy reward machine where reward functions (and not just numbers) are output by each RM transition. That was actually our original definition of a reward machine \cite{icml2018rms,camamacho2019ijcai}. However, following the same argument from Observation~\ref{prop:srm-are-rm}, we can see that any such Mealy reward machine can be encoded by a (Moore) reward machine using fewer reward functions (one per node instead of per edge). 
For reward machines that just output numbers, on the other hand, Mealy machines have the advantage of, in some cases, requiring fewer states to represent the same reward signal (also, their first output can depend on the input, unlike for a Moore machine).

\section{Exploiting the RM Structure in Reinforcement Learning}
\label{sec:rm4rl}

In this section, we describe a collection of RL approaches to learn policies for MDPRMs. We begin by describing a baseline that uses 
standard RL. We then discuss three approaches that exploit the information in the reward machine to facilitate learning. In all these cases, we include pseudo-code for their tabular implementation, describe how to extend them to work with deep RL, and discuss their convergence guarantees.

\subsection{The Cross-Product Baseline}

\begin{algorithm}[tb]
   \caption{The cross-product baseline using tabular Q-learning.}
   \label{alg:ql}
    \begin{algorithmic}[1]
    \STATE {\bfseries Input:} $S$, $A$, $\gamma \in (0,1]$, $\alpha \in (0,1]$, $\epsilon \in (0,1]$, $\mathcal{P}$, $L$, $U$, $u_0$, $F$, $\delta_u$, $\delta_r$.
    \STATE For all $s \in S$, $u \in U$, and $a \in A$, initialize $\qapprox(s,u,a)$ arbitrarily
    \FOR{$l\leftarrow 0$ \TO num\_episodes} 
        \STATE Initialize $u \leftarrow u_0$ and $s \leftarrow$ EnvInitialState()
        \WHILE{$s$ is not terminal \AND $u \not\in F$} 
            \STATE Choose action $a$ from $(s,u)$ using policy derived from $\qapprox$ (e.g., $\epsilon$-greedy)
            \STATE Take action $a$ and observe the next state $s'$
            \STATE Compute the reward $r \leftarrow \delta_r(u)(s,a,s')$ and next RM state $u' \leftarrow \delta_u(u,L(s,a,s'))$
            \IF{$s'$ is terminal \OR $u' \in F$} 
                \STATE $\qapprox(s,u,a) \xleftarrow{\alpha} r$
            \ELSE
                \STATE $\qapprox(s,u,a) \xleftarrow{\alpha} r + \gamma \max_{a'\in A}{\qapprox(s',u',a')}$
            \ENDIF
            \STATE Update $s \leftarrow s'$ and $u \leftarrow u'$
        \ENDWHILE
    \ENDFOR
\end{algorithmic}
\end{algorithm}

As discussed in Observation~\ref{obs:cross-mdp}, MDPRMs are regular MDPs when considering the cross-product between the environment states $S$ and the reward machine states $U$. As such, any RL algorithm can be used to learn a policy $\pi(a|s,u)$ -- including tabular RL methods and deep RL methods. If the RL algorithm is guaranteed to converge to optimal policies, then it will also find optimal policies for the MDPRM.

As a concrete example, Algorithm~\ref{alg:ql} shows pseudo-code for solving MDPRMs using tabular Q-learning. The only difference with standard Q-learning is that it also keeps track of the current RM state $u$ and learns Q-values over the cross-product $\qapprox(s,u,a)$. This allows the agent to consider the current environment state $s$ and RM state $u$ when selecting the next action $a$. 
Given the current experience $\tuple{s,u,a,r,s',u'}$, where $\tuple{s',u'}$ is the cross-product state reached after executing action $a$ in state $\tuple{s,u}$ and receiving a reward $r$, the Q-value $\qapprox(s,u,a)$ is updated as follows: $\qapprox(s,u,a) \xleftarrow{\alpha} r + \gamma \max_{a'}{\qapprox(s',u',a')}$. 

While this method has the advantage of allowing for the use of any RL method to solve MDPRMs, it does not exploit the information exposed by the RM. Below, we discuss different approaches that make use of such information to learn policies for MDPRMs faster.

\subsection{Counterfactual Experiences for Reward Machines (CRM)}
\label{sec:crm}

Our first method to exploit the information from the reward machine is called \emph{counterfactual experiences for reward machines (CRM)}. This approach also learns policies over the cross-product $\pi(a|s,u)$, but
uses counterfactual reasoning to generate \textit{synthetic} experiences. These experiences can then be used by an off-policy learning method, such as tabular Q-learning, DQN, or DDPG, to learn a policy $\pi(a|s,u)$ faster.

\begin{algorithm}[tb]
   \caption{Tabular Q-learning with counterfactual experiences for RMs (CRM).}
   \label{alg:crm}
    \begin{algorithmic}[1]
    \STATE {\bfseries Input:} $S$, $A$, $\gamma \in (0,1]$, $\alpha \in (0,1]$, $\epsilon \in (0,1]$, $\mathcal{P}$, $L$, $U$, $u_0$, $F$, $\delta_u$, $\delta_r$.
    \STATE For all $s \in S$, $u \in U$, and $a \in A$, initialize $\qapprox(s,u,a)$ arbitrarily
    \FOR{$l\leftarrow 0$ \TO num\_episodes} 
        \STATE Initialize $u \leftarrow u_0$ and $s \leftarrow$ EnvInitialState()
        \WHILE{$s$ is not terminal \AND $u \not\in F$} 
            \STATE Choose action $a$ from $(s,u)$ using policy derived from $\qapprox$ (e.g., $\epsilon$-greedy)
            \STATE Take action $a$ and observe the next state $s'$
            \STATE Compute the reward $r \leftarrow \delta_r(u)(s,a,s')$ and next RM state $u' \leftarrow \delta_u(u,L(s,a,s'))$
            \STATE Set experience $\leftarrow \{\tuple{s,\bar{u},a,\delta_r(\bar{u})(s,a,s'),s',\delta_u(\bar{u},L(s,a,s'))} \ \ |\ \  \forall \bar{u} \in U\}$            \FOR{$\tuple{s,\bar{u},a,\bar{r},s',\bar{u}'} \in $ experience}
                \IF{$s'$ is terminal \OR $\bar{u}' \in F$} 
                    \STATE $\qapprox(s,\bar{u},a) \xleftarrow{\alpha} \bar{r}$
                \ELSE
                    \STATE $\qapprox(s,\bar{u},a) \xleftarrow{\alpha} \bar{r} + \gamma \max_{a'\in A}{\qapprox(s',\bar{u}',a')}$
                \ENDIF
            \ENDFOR
            \STATE Update $s \leftarrow s'$ and $u \leftarrow u'$
        \ENDWHILE
    \ENDFOR
\end{algorithmic}
\end{algorithm}

Suppose that the agent performed action $a$ when in the cross-product state $\tuple{s,u}$ and then reached state $\tuple{s',u'}$ while receiving a reward of $r$. For every RM state $\bar{u} \in U$, we know that if the agent had been at $\bar{u}$ when $a$ caused the transition from $s$ to $s'$, then the next RM state would have been $\bar{u}' = \delta_u(\bar{u},L(s,a,s'))$ and the agent would have received a reward of $\bar{r} = \delta_r(\bar{u})(s,a,s')$. This is the key idea behind CRM. What CRM does is that, after every action, instead of feeding only the actual experience $\tuple{s,u,a,r,s',u'}$ to the RL agent, it feeds one experience per RM state, i.e., the following set of experiences: 
\begin{align}
    \{\tuple{s,\bar{u},a,\delta_r(\bar{u})(s,a,s'),s',\delta_u(\bar{u},L(s,a,s'))} \ |\  \forall \bar{u} \in U\}.
\end{align} 

Notice that incorporating CRM into an off-policy learning method is trivial. For instance, Algorithm~\ref{alg:crm} shows that CRM can be included in tabular Q-learning by adding two lines of code (lines 9 and 10). CRM can also be easily adapted to other off-policy methods by adjusting how the generated experiences are then used for learning. For example, for both DQN and DDPG, the counterfactual experiences would simply be added to the experience replay buffer and then used for learning as is typically done with these algorithms.

In Section \ref{sec:results}, we will show empirically that CRM can be very effective at learning policies for MDPRMs. The intuition behind its good performance is that CRM allows the agent to reuse experience to learn the right behaviour at different RM states. Consider, for instance, the RM from Figure~\ref{fig:officeright}, which rewards the agent for delivering a coffee to the office.
Suppose that the agent gets to the office before getting the coffee. The cross-product baseline would use that experience to learn that going to the office is not an effective way to get coffee. In contrast, CRM would also use that experience to learn how to get to the office. As such, CRM will already have made progress towards learning a policy that will finish the task as soon as it finds the coffee, since it will already have experience about how to get to the office. Importantly, CRM also converges to optimal policies when combined with Q-learning:

\begin{theorem}
    \label{the:convergence}
    Given an MDPRM $\mathcal{T}=\tuple{S, A, p, \gamma, \mathcal{P}, L, U, u_0, F, \delta_u, \delta_r}$, CRM with tabular Q-learning converges to an optimal policy for $\mathcal{T}$ in the limit (as long as every state-action pair is visited infinitely often).
\end{theorem}
\begin{proof}
The convergence proof provided by \citeA{watkins1992q} for tabular Q-learning applies directly to the case of CRM when we consider that each experience produced by CRM is still sampled according its transition probability $p(\tuple{s',u'}|\tuple{s,u},a) = p(s'|s,a)$.
\end{proof}

\subsubsection{Q-Learning for Reward Machines (QRM)}

In the original paper on reward machines \shortcite{icml2018rms}, we proposed \emph{Q-learning for reward machines (QRM)} as a way to exploit reward machine structure. CRM and QRM are both based on the same fundamental idea: to reuse experience to simultaneously learn optimal behaviours for the different RM states. The key difference is that CRM learns a single Q-value function $\qapprox(s, u, a)$ that takes into account both the environment and RM state, while QRM learns a separate Q-value function $\qapprox_u$ for each RM state $u \in U$. Formally, QRM uses any experience $\tuple{s,a,s'}$ to update each $\qapprox_u$ as follows:
\begin{align}
    \qapprox_u(s,a) \xleftarrow{\alpha} \delta_r(u)(s,a,s') + \gamma \max_{a'\in A}{\qapprox_{\delta_u(u,L(s,a,s'))}(s',a')} 
\end{align}

Note that QRM will behave identically to Q-learning with CRM in the tabular case. Intuitively, this is because the Q-value function $\qapprox(s, u, a)$ of CRM can be partitioned by reward machine state, to yield reward machine state specific Q-value functions just as in QRM. The corresponding updates will thus be identical.

However, QRM and CRM can differ when using function approximation. Consider, for example, the case of using Deep Q-Networks for function approximation. While the definition of QRM suggests the use of a separate Q-network for each RM state, CRM will learn a single Q-network that covers all the RM states. We note that the need for separate Q-networks makes the implementation of deep QRM fairly complex. In contrast, combining CRM with DQN or DDPG is trivial: it only requires adding the reward machine state to the experiences when they are being added to the experience replay buffer. 
We will also see that CRM performs slightly better than QRM in our deep RL experiments.

\subsection{Hierarchical Reinforcement Learning for Reward Machines (HRM)}

Our second approach to exploit the structure of a reward machine is based on \emph{hierarchical reinforcement learning (HRL)}, in particular, the options framework \shortcite{sutton1999between}. The overall idea is to decompose the problem into subproblems, called \emph{options}, that are potentially simpler to solve. 
Formally, an option is a triple $\tuple{\mathcal{I},\pi,\beta}$, where $\mathcal{I}$ is the initiation set (the subset of the state space in which the option can be started), $\pi$ is the policy that chooses actions while the option is being followed, and $\beta$ gives the probability that the option will terminate in each state. 

In our case, the agent will learn a set of options for the cross-product MDP, that focus on learning how to move from one RM state to another RM state. Then, a higher-level policy will learn how to select among those options in order to collect reward.

As an example, consider the reward machine shown in Figure~\ref{fig:fsm2}. This machine rewards the agent when it delivers a coffee and the mail to the office. To do so, the agent might first get the coffee, then the mail, and go to the office. Alternatively, the agent might get the mail first, then the coffee, and then go to the office. For this reward machine, our hierarchical RL method will learn one option per edge for a total of nine options: five for the transitions between different RM states, and four corresponding to the self-loop transitions that remain in the same RM state. That is, the method will learn one policy to get a coffee before getting the mail (moving from $u_0$ to $u_1$), one policy to get the mail before getting a coffee (moving from $u_0$ to $u_2$), one policy to get a coffee (moving from $u_2$ to $u_3$), one policy to get the mail (moving from $u_1$ to $u_3$), one policy to go to the office (moving from $u_3$ to the terminal state), and one policy for each possible transition that remains in the same RM state (from $u_i$ to $u_i$ for all $i$). The role of the high-level policy is to decide which option to execute next among these available options. For instance, if the RM state is in $u_0$, the high-level policy will decide whether to get a coffee first (moving to $u_1$), the mail (moving to $u_2$), or neither (staying in $u_0$). To make this decision, it will consider the current environment state and, thus, it can learn to get the coffee or mail depending on which one is closer to the agent.

More generally, we learn one option for each pair of RM states $\tuple{u,u_t}$ that are connected in the RM, including self-loop edges where $u=u_t$.
We will name the options with the pairs of RM states $\tuple{u,u_t}$ that they correspond to. This means that the set of options is $\mathcal{A} = \{\tuple{u,\delta_u(u,\sigma)} \ |\ u \in U, \sigma \in 2^\mathcal{P} \}$.
The option $\tuple{u,u_t}$ will have its initiation set defined to contain all the states in the cross-product MDP where the RM state is $u$: 
$\mathcal{I}_{\tuple{u,u_t}}=\{\tuple{s,u}:s\in S\}$. The termination condition is then defined as follows:
\begin{align}
\beta_{\tuple{u,u_t}}(s',u')=\begin{cases}
1&\text{ if }u'\ne u\text{ or }s'\text{ is terminal}\\
0&\text{ otherwise}
\end{cases}
\end{align}
That is, the option $\tuple{u,u_t}$ terminates (deterministically) when a new RM state is reached or a terminal environment state is reached. Since $\tuple{u,u_t}$ can only be executed when the RM state is $u$, its policy can be described in terms of the environment state $s$ only. As such, we refer to the option policy as $\pi_{u,u_t}(a|s)$.

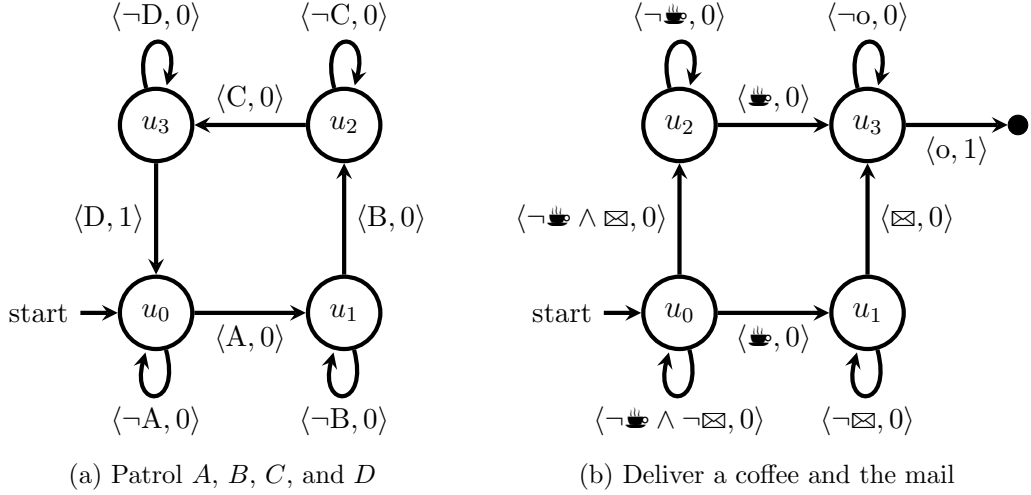
\begin{figure}
    \centering
    \begin{subfigure}{0.45\columnwidth}
        \centering
        \begin{tikzpicture}[node distance=2cm,on grid,every initial by arrow/.style={ultra thick,->, >=stealth}]
  \node[ultra thick,state,initial] (q_0) at (0,0)                {$u_0$};
  \node[ultra thick,state]         (q_1) at (2.5,0) {$u_1$};
  \node[ultra thick,state]         (q_2) at (2.5,2.5) {$u_2$};
  \node[ultra thick,state]         (q_3) at (0,2.5) {$u_3$};
  \path[ultra thick,->, >=stealth] (q_0) edge node [below] {$\tuple{\text{A},0}$} (q_1);
  \path[ultra thick,->, >=stealth] (q_0) edge [loop below] node {$\tuple{\neg \text{A},0}$} ();
  \path[ultra thick,->, >=stealth] (q_1) edge node [right] {$\tuple{\text{B},0}$} (q_2);
  \path[ultra thick,->, >=stealth] (q_1) edge [loop below] node {$\tuple{\neg \text{B},0}$} ();
  \path[ultra thick,->, >=stealth] (q_2) edge node [above] {$\tuple{\text{C},0}$} (q_3);
  \path[ultra thick,->, >=stealth] (q_2) edge [loop above] node {$\tuple{\neg \text{C},0}$} ();
  \path[ultra thick,->, >=stealth] (q_3) edge node [left] {$\tuple{\text{D},1}$} (q_0);
  \path[ultra thick,->, >=stealth] (q_3) edge [loop above] node {$\tuple{\neg \text{D},0}$} ();
\end{tikzpicture}
        \subcaption{Patrol $A$, $B$, $C$, and $D$}
        \label{fig:fsm1}
    \end{subfigure}\begin{subfigure}{0.5\columnwidth}
        \centering
        \begin{tikzpicture}[node distance=2cm,on grid,every initial by arrow/.style={ultra thick,->, >=stealth}]
  \node[ultra thick,state,initial] (q_0) at (0,0) {\large$u_0$};
  \node[ultra thick,state]         (q_1) at (2.5,0) {\large$u_1$};
  \node[ultra thick,state]         (q_2) at (0,2.5) {\large$u_2$};
  \node[ultra thick,state]         (q_3) at (2.5,2.5) {\large$u_3$};
  \node[circle,draw=black,minimum size=0.26cm,inner sep=0pt,fill=black] (t1) at (4.5,2.5)  {};
  \path[ultra thick,->, >=stealth] (q_0) edge node [below] {$\tuple{\text{\coffee},0}$} (q_1);
  \path[ultra thick,->, >=stealth] (q_0) edge [loop below] node {$\tuple{\neg \text{\coffee} \wedge \neg \text{\mail}, 0}$} ();
  \path[ultra thick,->, >=stealth](q_1) edge node [right] {$\tuple{\text{\mail},0}$} (q_3);
  \path[ultra thick,->, >=stealth] (q_1) edge [loop below] node {$\tuple{\neg \text{\mail},0}$} ();
  \path[ultra thick,->, >=stealth] (q_3) edge node [below] {$\tuple{\text{o},1}$} (t1);
  \path[ultra thick,->, >=stealth] (q_3) edge [loop above] node {$\tuple{\neg \text{o},0}$} ();
  \path[ultra thick,->, >=stealth] (q_0) edge node [left] {$\tuple{\neg \text{\coffee} \wedge \text{\mail},0}$} (q_2);
  \path[ultra thick,->, >=stealth] (q_2) edge node [above] {$\tuple{\text{\coffee},0}$} (q_3);
  \path[ultra thick,->, >=stealth] (q_2) edge [loop above] node {$\tuple{\neg \text{\coffee},0}$} ();
\end{tikzpicture}

  
        \subcaption{Deliver a coffee and the mail}
        \label{fig:fsm2}
    \end{subfigure}
    \caption{Two more reward machines for the office gridworld}
    \label{fig:fsm}
\end{figure}

Since the objective of option $\pi_{u,u_t}$ is to induce the reward machine to transition to $u_t$ as soon as possible, we train the option policy $\pi_{u,u_t}(a|s)$ using the following reward function:
\begin{align}
r_{u,u_t}(s,a,s')=
\begin{cases}
\delta_r(u)(s,a,s') + r^+ &\text{ if } u_t \neq u \text{ and } u_t = \delta_u(u,L(s,a,s'))\\
\delta_r(u)(s,a,s') + r^-& \text{ if } u_t \neq u \text{ and } u_t \neq \delta_u(u,L(s,a,s'))\\
\delta_r(u)(s,a,s') &\text{ otherwise}
\end{cases}
\end{align}
where $r^+$ and $r^-$ are hyperparameters. This reward function states that the policy $\pi_{u,u_t}$ gets the usual reward $\delta_r(u)(s,a,s')$ plus an additional reward of $r^+$ (a bonus) when the transition $(s,a,s')$ causes the RM state to move from $u$ to $u_t$ (unless $u=u_t)$ and an additional reward of $r^-$ (a penalization) when it causes the RM state to move from $u$ to some other state $\bar{u} \not\in \{u,u_t\}$.
Crucially, the policies for all the options will be learned simultaneously, using off-policy RL and counterfactual experience generation.

The high-level policy decides which option to execute next from the set of available options. The policy $\pi(u_t|s,u)$ being learned will determine the probability of executing each option $\tuple{u,u_t}\in \mathcal{A}$ given the current environment state $s$ and RM state $u$. We note that this high-level policy can only choose among the options that start at the current RM state $u$. To train this policy, we use the reward coming from the reward machine.

Algorithm~\ref{alg:hrm} shows pseudocode for this approach, which we call \emph{hierarchical reinforcement learning for reward machines (HRM)}, when using tabular Q-learning. Tabular Q-learning could be replaced by any other off-policy method such as DQN or DDPG. The algorithm begins by initializing one Q-value estimate $\qapprox(s,u,u_t)$ for the high-level policy and one Q-value estimate $\qapprox_{u,u_t}(s,a)$ for each option $\tuple{u,u_t} \in \mathcal{A}$. At every step, the agent first checks if a new option has to be selected and then does so using $\qapprox(s,u,u_t)$ (lines 8--10). This option takes the control of the agent until it reaches a terminal transition. The current option selects the next action $a \in A$, executes it, and reaches the next state $s'$ (lines 11--12). The experience $(s,a,s')$ is used to compute the next RM state $u'=\delta_u(u,L(s,a,s'))$ and reward $r=\delta_r(u)(s,a,s')$ (line 13), and also to update the option policies by giving a reward of $r_{\bar{u},\bar{u}_t}(s,a,s')$ to each option $\tuple{\bar{u},\bar{u}_t}\in \mathcal{A}$  (lines 14--18). Finally, the high-level policy is updated when the option ends (lines 19--24) and the loop starts over from $\tuple{s',u'}$.

\begin{algorithm}[tb]
   \caption{Tabular hierarchical RL for reward machines (HRM).}
   \label{alg:hrm}
    \begin{algorithmic}[1]
    \STATE {\bfseries Input:} $S$, $A$, $\gamma \in (0,1]$, $\alpha \in (0,1]$, $\epsilon \in (0,1]$, $\mathcal{P}$, $L$, $U$, $u_0$, $F$, $\delta_u$, $\delta_r$.
    \STATE $\mathcal{A}(u) \leftarrow \{u_t \ |\ u_t=\delta_u(u,\sigma) \text{ for some } u_t \in U \cup F, \sigma \in 2^\mathcal{P} \}$ for all $u \in U$
    \STATE For all $s \in S$, $u \in U$, and $u_t \in \mathcal{A}(u)$, initialize the high-level $\qapprox(s,u,u_t)$ arbitrarily
    \STATE For all $s \in S$, $u \in U$, $u_t \in \mathcal{A}(u)$, and $a\in A$, initialize option $\qapprox_{u,u_t}(s,a)$ arbitrarily
    \FOR{$l \leftarrow 0$ \TO num\_episodes} 
        \STATE Initialize $u \leftarrow u_0$, $s \leftarrow$ EnvInitialState(), and $u_t \leftarrow \emptyset$
        \WHILE{$s$ is not terminal \AND $u \not\in F$} 
            \IF{$u_t = \emptyset$}
                \STATE Choose option $u_t \in \mathcal{A}(u)$ using policy derived from $\qapprox$ (e.g., $\epsilon$-greedy)
                \STATE Set $r_t \leftarrow 0$ and $t \leftarrow 0$
            \ENDIF
            \STATE Choose action $a$ from $s$ using policy derived from $\qapprox_{u,u_t}$ (e.g., $\epsilon$-greedy)
            \STATE Take action $a$ and observe the next state $s'$
            \STATE Compute the reward $r \leftarrow \delta_r(u)(s,a,s')$ and next RM state $u' \leftarrow \delta_u(u,L(s,a,s'))$
            \FOR{$\bar{u} \in U, \bar{u}_t \in \mathcal{A}(\bar{u})$}
                \IF{$\delta_u(\bar{u},L(s,a,s')) \neq u$ or $s'$ is terminal} 
                    \STATE $\qapprox_{\bar{u},\bar{u}_t}(s,a) \xleftarrow{\alpha} r_{\bar{u},\bar{u}_t}(s,a,s')$
                \ELSE
                    \STATE $\qapprox_{\bar{u},\bar{u}_t}(s,a) \xleftarrow{\alpha} r_{\bar{u},\bar{u}_t}(s,a,s') + \gamma \max_{a' \in A}{\qapprox_{\bar{u},\bar{u}_t}(s',a')}$
                \ENDIF
            \ENDFOR
            \IF{$s'$ is terminal \OR $u' \neq u$}
                \IF{$s'$ is terminal \OR $u'\in F$} 
                    \STATE $\qapprox(s,u,u_t) \xleftarrow{\alpha} r_t + \gamma^t r$
                \ELSE
                    \STATE $\qapprox(s,u,u_t) \xleftarrow{\alpha} r_t + \gamma^t r + \gamma^{t+1} \max_{u_t'\in \mathcal{A}(u')}{\qapprox(s',u',u_t')}$
                \ENDIF
                \STATE Set $u_t \leftarrow \emptyset$
            \ENDIF
            \STATE Update $s \leftarrow s'$ and $u \leftarrow u'$
            \STATE Update $r_t \leftarrow r_t + \gamma^t r$
            \STATE Update $t \leftarrow t + 1$
        \ENDWHILE
    \ENDFOR
\end{algorithmic}
\end{algorithm}

HRM can be very effective at quickly learning good policies for MDPRMs. Its strength comes from its ability to learn policies for all of the options simultaneously through off-policy learning. However, it might converge to sub-optimal solutions, even in the tabular case. This is because the option-based approach is myopic: the learned option policies will always try to transition as quickly as possible without considering how that will affect performance after the transition occurs. An example of this behaviour is shown in Figure~\ref{fig:office}. The task consists of delivering a coffee to the office. As such, the optimal high-level policy will correctly learn to go for the coffee at state $u_0$ and then go to the office at state $u_1$. However, the optimal option policy for getting the coffee will move to the closest coffee station (following the sub-optimal red path in Figure~\ref{fig:officeleft}) because (i) that option gets a large reward when it reaches the coffee, and (ii) optimal policies will always prefer to collect such a reward \emph{as soon as} possible. As a result, HRM will converge to a sub-optimal policy.

We note that HRM can use prior knowledge about the environment to prune useless options. For example, in our experiments we do not learn options for the self-loops, since no optimal high-level policy would need to self-loop in our domains.
We also do not learn options that lead to ``bad" terminal states, such as breaking decorations in Figure~\ref{fig:officeright}. 

\subsection{Automated Reward Shaping (RS)}

Our last method for exploiting reward machines builds on the idea of \emph{potential-based reward shaping} \shortcite{Ng1999shaping}. The intuition behind reward shaping is that some reward functions are easier to learn policies for than others, even if those functions have the same optimal policy. Typically, this involves providing some intermediate rewards as the agent gets closer to completing the task. To that end, \citeA{Ng1999shaping} formally showed that given any MDP $\mathcal{M}=\mdp$ and function $\Phi: S \to \reals$, changing the reward function of $\mathcal{M}$ to 
\begin{equation}
\label{eq:rs}
r'(s,a,s') = r(s,a,s') + \gamma \Phi(s') - \Phi(s)    
\end{equation}
will not change the set of optimal policies. Thus, if we find a function $\Phi$ -- referred to as a \emph{potential function} -- that allows us to learn optimal policies more quickly, we are guaranteed that the found policies are still optimal with respect to the original reward function.

In this section, we consider the use of \emph{value iteration} over the RM states as a way to compute a potential function. Intuitively, the idea is to approximate the expected discounted return of being in any RM state by treating the RM itself as an MDP. As a result, a potential will be assigned to each RM state over which equation \eqref{eq:rs} will be used to define a shaped reward function that will encourage the agent to make progress towards solving the task. This method works only for simple reward machines since it does not use information from the environment states $S$ and actions $A$. 

Formally, given a simple RM $\tuple{U,u_0, F,\delta_u,\delta_r}$, we construct an MDP $\mathcal{M}=\mdp$, where $S = U \cup F$, $A = 2^\mathcal{P}$, $r(u,\sigma,u') = \delta_r(u,\sigma)$ if $u\in U$ (zero otherwise), $\gamma < 1$, and
\begin{align}
p(u'|u,\sigma)=
    \begin{cases}
    1 &\text{ if } u \in F \text{ and } u' = u\\
    1 &\text{ if } u \in U \text{ and } u' = \delta_u(u,\sigma)\\
    0 &\text{ otherwise}
    \end{cases}
\end{align}
Intuitively, this is an MDP where every transition in the RM corresponds to a deterministic action. We can then compute the value of each state $v^*(u) = \max_{\sigma}{q^*(u,\sigma)}$ in this MDP when using the optimal policy, using a method such as value iteration. This is shown for a given $U$, $F$, $\mathcal{P}$, $\delta_u$, $\delta_r$, and $\gamma$ in Algorithm~\ref{alg:rs}. Here, the computed state-value estimates $v$ are guaranteed to be equal to $v^*$ by the well-known convergence properties of value iteration.

\begin{algorithm}[tb]
    \caption{Value iteration for automated reward shaping}
    \label{alg:rs}
    \begin{algorithmic}[1]
    \STATE {\bfseries Input:} $U$, $F$, $\mathcal{P}$, $\delta_u$, $\delta_r$, $\gamma$
    \FOR{$u \in U \cup F$} 
        \STATE $v(u) \leftarrow 0$ \COMMENT{initializing v-values}
    \ENDFOR
    \STATE $e \leftarrow 1$ 
    \WHILE{$e > 0$} \label{alg:rs_terminate}
        \STATE $e \leftarrow 0$ 
        \FOR{$u \in U$} 
            \STATE $v' \leftarrow \max\{\delta_r(u,\sigma) + \gamma v(\delta_u(u,\sigma))\ |\ \forall \sigma \in 2^{\mathcal{P}}\}$
            \STATE $e = \max\{e, |v(u)-v'|\}$
            \STATE $v(u) \leftarrow v'$
        \ENDFOR
    \ENDWHILE
    \RETURN $v$
\end{algorithmic}
\end{algorithm}

Once we have computed $v^*$, we then define the potential function as $\Phi(s,u) = -v^*(u)$ for every environment state $s$ and RM state $u$. As we will see below, the use of negation encourages the agent to transition towards RM states that correspond to task completion. 

To make this approach clearer, consider the example task of \textit{delivering coffee to the office while avoiding decorations} from Figure~\ref{fig:officeright}. What makes this task difficult for an RL agent is the sparsity of the reward. The agent only gets a +1 reward by completing the whole task. Recall that one of the typical goals of reward shaping is to provide some intermediate rewards as the agent gets closer to completing the task. In this case, passing this simple reward machine through Algorithm~\ref{alg:rs} with $\gamma = 0.9$ results in the potential-based function shown in Figure~\ref{fig:rs}. In the figure, nodes represent RM states, and each state has been labelled with the computed potential in red (the potential of terminal states is always zero). Each transition has also been labelled by a pair $\tuple{c,r+\mathit{rs}}$, where $c$ is a logical condition to transition between the states, $r$ is the reward that the agent receives for the transition according to $\delta_r$, and $\mathit{rs}$ is the extra reward given by equation \eqref{eq:rs}. Note that, with reward shaping, the agent is given a reward of $0.09$ for self-looping before getting a coffee and a reward of $0.1$ for self-looping after getting the coffee. This gives an incentive for collecting coffee and, as such, making progress on the overall task. In addition, we know that the optimal policy is preserved since we are using equation \eqref{eq:rs} to shape the rewards.

\begin{figure}
    \centering
    \newcommand{\rs}[1]{{\bf\color{red}+#1}}\begin{tikzpicture}[node distance=2cm,on grid,every initial by arrow/.style={ultra thick,->, >=stealth}]
  \node[ultra thick,state,initial,minimum size=1.02cm] (u_0) at (0, 0) {\bf\color{red}-0.9};
  \node[ultra thick,state,minimum size=1.02cm]         (u_1) at (4.5,0) {\bf\color{red}-1.0};
  \node[circle,draw=black,minimum size=0.26cm,inner sep=0pt,fill=black]         (u_2) at (0,-1.5) {};
  \node[circle,draw=black,minimum size=0.26cm,inner sep=0pt,fill=black]         (u_3) at (9,0) {};
  \node[circle,draw=black,minimum size=0.26cm,inner sep=0pt,fill=black]         (u_4) at (4.5,-1.5) {};
  
  \path[ultra thick,->, >=stealth] (u_0) edge node [above] {$\tuple{\text{\coffee} \wedge \neg \text{\decoration},0\rs{0}}$} (u_1);
  \path[ultra thick,->, >=stealth] (u_0) edge [loop above] node [right] {\ $\tuple{\neg \text{\coffee} \wedge \neg \text{\decoration},0\rs{0.09}}$} ();
  \path[ultra thick,->, >=stealth] (u_0) edge node [right] {$\tuple{\text{\decoration},0\rs{0.9}}$} (u_2);
  \path[ultra thick,->, >=stealth] (u_1) edge node [above] {$\tuple{\text{o} \wedge \neg \text{\decoration},1\rs{1}}$} (u_3);
  \path[ultra thick,->, >=stealth] (u_1) edge [loop above] node [right] {$\tuple{\neg \text{o} \wedge \neg \text{\decoration},0\rs{0.1}}$} ();
  \path[ultra thick,->, >=stealth] (u_1) edge node [right] {$\tuple{ \text{\decoration},0\rs{1}}$} (u_4);
\end{tikzpicture}
    \caption{Reward shaping example with $\gamma = 0.9$.}
    \label{fig:rs}
\end{figure}
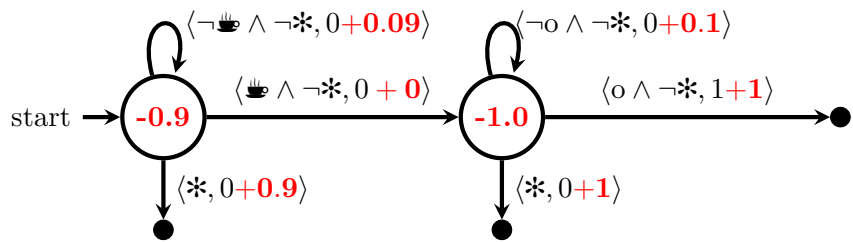

Finally, we would like to make two observations about this approach. First, the potentials of all terminal states are set to zero and, thus, moving to any terminal state will give a positive reward to the agent (as long as the RM has non-negative rewards only), even if that terminal state is a \textit{bad} terminal state. For instance, the agent will receive a reward of $0.9$ or $1.0$ when breaking a decoration in Figure~\ref{fig:rs}. Unfortunately, we cannot define different potentials for \textit{good} and \textit{bad} terminal states because potential-based reward shaping works under the assumption that the potentials of all terminal states are set to the same value \cite{Ng1999shaping}. The second observation is that we are computing the potentials by solving a \emph{deterministic} MDP using value iteration. However, there exist faster methods to solve deterministic MDPs \shortcite<e.g.,>{post2015simplex,bertram2018fast}.

\section{Experimental Evaluation}
\label{sec:results}

In this section, we provide an empirical evaluation of our methods in domains with a variety of characteristics: discrete states, continuous states, and continuous action spaces. Some domains include multitask learning and single task learning. Most of the tasks considered have been expressed using simple reward machines, though those used in Section \ref{sec:res-cheetah} require the full formulation as we describe below. As a brief summary, our results show the following:
\begin{enumerate}
    \itemsep0em 
    \item CRM and HRM outperform the cross-product baselines in all our experiments.
    \item CRM converges to the best policies in all but one experiment.
    \item HRM tends to initially learn faster than CRM but converges to suboptimal policies.
    \item The gap between CRM/HRM and the cross-product baseline increases when learning in a multitask setting.
    \item Reward shaping helps in discrete domains but it does not in continuous domains.
\end{enumerate}

\subsection{Results on Discrete Domains}

We evaluated our methods on two gridworlds. Since these are tabular domains, we use tabular Q-learning as the core off-policy learning method for all the approaches. Specifically, we tested Q-learning alone over the cross-product (QL), Q-learning with reward shaping (QL+RS), Q-learning with counterfactual experiences (CRM), Q-learning with CRM and reward shaping (CRM+RS), and our hierarchical RL method with and without reward shaping (HRM and HRM+RS). We do not report results for QRM since it is equivalent to CRM in tabular domains. We use $\epsilon=0.1$ for exploration, $\gamma=0.9$, and $\alpha=0.5$. We also used optimistic initialization of the Q-values by setting the initial Q-value of any state-action pair to be $2$. For the hierarchical RL methods, we use $r^+ = 1$ and $r^-=0$.

The first domain is the office world described in Section~\ref{sec:reward-machines}  and Figure~\ref{fig:officeleft}. This is a multitask domain, consisting of the four tasks described in Table~\ref{tab:owt}. 
We begin by evaluating how long it takes the agents to learn a policy that can solve those four tasks. The agent will iterate through the tasks, changing from one to the next at the completion of each episode. Note that CRM and HRM are good fits for this problem setup because they can use experience from solving one task to update the policies for solving the other tasks. 

\begin{table}
\caption{Tasks for the office world.}
\label{tab:owt}
\centering
\begin{tabular}{rl}
    \# & Description \\
    \hline
    1& deliver coffee to the office without breaking any decoration\\
    2& deliver mail to the office without breaking any decoration\\
    3& patrol locations $A$, $B$, $C$, and $D$, without breaking any decoration \\
    4& deliver a coffee and the mail to the office without breaking any decoration\\
\end{tabular}
\end{table}

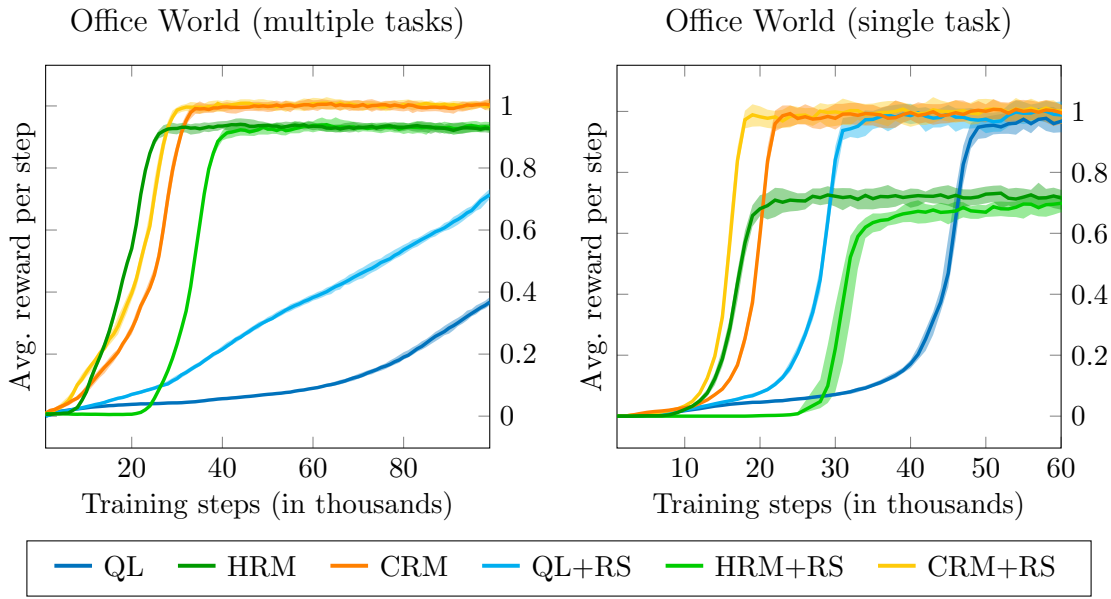
\begin{figure}
    \centering
    \begin{subfigure}{0.49\columnwidth}
        \centering
        \begin{tikzpicture}%

  \begin{axis}[%
    scale only axis,
    xlabel={Training steps (in thousands)},
    ylabel={Avg. reward per step},
    title style={font=\large},
    title={Office World (multiple tasks)},
    every axis x label/.style={at={(current axis.south)},above=-11mm},
    every axis y label/.append style={at={(current axis.west)},above=0mm},
    yticklabel pos=right,
    scaled x ticks = false,
    enlarge x limits=false,
    scale= 0.7
  ]

    \percentiles{data_V4/office-ql.txt}{0}{1}{2}{3}{cyan!50!blue}{0.4}

    \percentiles{data_V4/office-ql-rs.txt}{0}{1}{2}{3}{cyan}{0.4}

    \percentiles{data_V4/office-crm-rs.txt}{0}{1}{2}{3}{yellow!80!red}{0.4}

    \percentiles{data_V4/office-crm.txt}{0}{1}{2}{3}{orange}{0.4}

    \percentiles{data_V4/office-hrm-rs.txt}{0}{1}{2}{3}{green!80!black}{0.4}

    \percentiles{data_V4/office-hrm.txt}{0}{1}{2}{3}{green!60!black}{0.4}

  \end{axis}
\end{tikzpicture}%
    \end{subfigure}
    \begin{subfigure}{0.49\columnwidth}
        \centering
        \begin{tikzpicture}%

  \begin{axis}[%
    scale only axis,
    xlabel={Training steps (in thousands)},
    ylabel={Avg. reward per step},
    title style={font=\large},
    title={Office World (single task)},
    every axis x label/.style={at={(current axis.south)},above=-11mm},
    every axis y label/.append style={at={(current axis.west)},above=0mm},
    yticklabel pos=right,
    scaled x ticks = false,
    enlarge x limits=false,
    xmax=60,
    scale= 0.7
  ]

    \percentiles{data_V4/office-single-ql.txt}{0}{1}{2}{3}{cyan!50!blue}{0.4}

    \percentiles{data_V4/office-single-ql-rs.txt}{0}{1}{2}{3}{cyan}{0.4}

    \percentiles{data_V4/office-single-crm-rs.txt}{0}{1}{2}{3}{yellow!80!red}{0.4}

    \percentiles{data_V4/office-single-crm.txt}{0}{1}{2}{3}{orange}{0.4}

    \percentiles{data_V4/office-single-hrm-rs.txt}{0}{1}{2}{3}{green!80!black}{0.4}

    \percentiles{data_V4/office-single-hrm.txt}{0}{1}{2}{3}{green!60!black}{0.4}

  \end{axis}
\end{tikzpicture}%
    \end{subfigure}
    
    \vspace{1mm}
    \begin{tikzpicture}%
\node[draw=black] {%
	\begin{tabular}{llllll}
	\entrysmall{cyan!50!blue}{QL}&
	\entrysmall{green!60!black}{HRM}&
	\entrysmall{orange}{CRM}&
	\entrysmall{cyan}{QL+RS}&
	\entrysmall{green!80!black}{HRM+RS}&
	\entrysmall{yellow!80!red}{CRM+RS}    	
	\end{tabular}};%
\end{tikzpicture}%
    \caption{Results on the office gridworld.}
    \label{fig:office-results}
\end{figure}

We ran 60 independent trials and report the average reward per step across the four tasks in Figure~\ref{fig:office-results}-left. We normalized the average reward per step to be 1 for an optimal policy (which we pre-computed using value iteration) and show the median performance over those 60 runs, as well as the 25th and 75th percentiles in the shadowed area. The results show that CRM and CRM+RS quickly learn to optimally solve all the tasks -- largely outperforming the cross-product baseline (QL). HRM also outperforms QL and initially learns faster than CRM, but converges to suboptimal policies. Finally, we note that adding reward shaping improves the performance of CRM and QL but decreases the performance of HRM. We are not certain about why reward shaping decreases the performance of HRM. Some further experiments showed that HRM+RS can slightly outperform HRM under certain combinations of values for $r^+$ and Q-value initialization. However, HRM outperforms HRM+RS under most hyperparameter selections, just as reported in Figure~\ref{fig:office-results}.

We also ran a single-task experiment in the office world. We took the hardest task available (task 4), and ran 60 new independent runs where the agent had to solve only that task. The results are shown in Figure~\ref{fig:office-results}-right. 
We note that CRM still outperforms the other methods, though the gap between CRM and QL decreased. 
HRM learns faster than CRM (without reward shaping), but is overtaken since it converges to a suboptimal policy.

\begin{table}
\caption{Tasks for the Minecraft domain. Each task is described as a sequence of events.}
\label{tab:minetasks}
\centering
\begin{tabular}{rll}
\# & Task name & Description\\\hline
1& make plank & get wood, use toolshed\\
2& make stick & get wood, use workbench\\
3& make cloth & get grass, use factory\\
4& make rope & get grass, use toolshed\\
5& make bridge & get iron, get wood, use factory\\
&& (the iron and wood can be gotten in any order)\\
6& make bed & get wood, use toolshed, get grass, use workbench\\
&& (the grass can be gotten at any time before using the workbench)\\
7& make axe & get wood, use workbench, get iron, use toolshed\\
&& (the iron can be gotten at any time before using the toolshed)\\
8& make shears & get wood, use workbench, get iron, use workbench\\
&& (the iron can be gotten at any time before using the workbench)\\
9& get gold & get iron, get wood, use factory, use bridge\\
&& (the iron and wood can be gotten in any order)\\
10& get gem & get wood, use workbench, get iron, use toolshed, use axe\\
&& (the iron can be gotten at any time before using the toolshed)
\end{tabular}
\end{table}

Our second tabular domain is the Minecraft-like gridworld introduced by \shortciteA{andreas2016sketches}. In this world, the grid contains raw materials that the agent can extract and use to make new objects. \citeauthor{andreas2016sketches} defined 10 tasks to solve in this world 
that consist of making an object by following a sequence of sub-goals (called a \emph{sketch}). For instance, the task \emph{make a bridge} consists of \emph{get iron},  \emph{get wood}, and \emph{use factory}. 
We note that sketches require a total ordering of the subgoals. For instance, we could define a sketch for the above example that would require the agent first \emph{get iron} and then \emph{get wood}, or we can define a sketch that requires the agent to \emph{get wood} and then \emph{get iron}. However, sketches do not allow for interleaving subtasks, which would allow an agent to decide to achieve different subgoals in whatever order they see fit. Reward machines do allow for such partially ordered tasks and, as such, we removed any unnecessary orderings when encoding these tasks as reward machines for our experiments. For example, the reward machine for the above task would allow for \emph{get wood} and \emph{get iron} to be done in any order. The original tasks and the partial orderings allowed in the new encodings are shown in Table~\ref{tab:minetasks}.

\begin{figure}
    \centering
    \begin{subfigure}{0.49\columnwidth}
        \centering
        \begin{tikzpicture}%

  \begin{axis}[%
    scale only axis,
    xlabel={Training steps (in thousands)},
    ylabel={Avg. reward per step},
    title style={font=\large},
    title={Craft World (multiple tasks)},
    every axis x label/.style={at={(current axis.south)},above=-11mm},
    every axis y label/.append style={at={(current axis.west)},above=0mm},
    yticklabel pos=right,
    scaled x ticks = false,
    enlarge x limits=false,
    scale= 0.7
  ]

    \percentiles{data_V4/craft-ql.txt}{0}{1}{2}{3}{cyan!50!blue}{0.4}

    \percentiles{data_V4/craft-ql-rs.txt}{0}{1}{2}{3}{cyan}{0.4}

    \percentiles{data_V4/craft-crm-rs.txt}{0}{1}{2}{3}{yellow!80!red}{0.4}

    \percentiles{data_V4/craft-crm.txt}{0}{1}{2}{3}{orange}{0.4}

    \percentiles{data_V4/craft-hrm-rs.txt}{0}{1}{2}{3}{green!80!black}{0.4}
    
    \percentiles{data_V4/craft-hrm.txt}{0}{1}{2}{3}{green!60!black}{0.4}

  \end{axis}
\end{tikzpicture}%
    \end{subfigure}
    \begin{subfigure}{0.49\columnwidth}
        \centering
        \begin{tikzpicture}%

  \begin{axis}[%
    scale only axis,
    xlabel={Training steps (in thousands)},
    ylabel={Avg. reward per step},
    title style={font=\large},
    title={Craft World (single task)},
    every axis x label/.style={at={(current axis.south)},above=-11mm},
    every axis y label/.append style={at={(current axis.west)},above=0mm},
    yticklabel pos=right,
    scaled x ticks = false,
    enlarge x limits=false,
    scale= 0.7
  ]

    \percentiles{data_V4/craft-single-ql.txt}{0}{1}{2}{3}{cyan!50!blue}{0.4}

    \percentiles{data_V4/craft-single-ql-rs.txt}{0}{1}{2}{3}{cyan}{0.4}

    \percentiles{data_V4/craft-single-crm-rs.txt}{0}{1}{2}{3}{yellow!80!red}{0.4}

    \percentiles{data_V4/craft-single-crm.txt}{0}{1}{2}{3}{orange}{0.4}

    \percentiles{data_V4/craft-single-hrm-rs.txt}{0}{1}{2}{3}{green!80!black}{0.4}

    \percentiles{data_V4/craft-single-hrm.txt}{0}{1}{2}{3}{green!60!black}{0.4}

  \end{axis}
\end{tikzpicture}%
    \end{subfigure}
        
    \vspace{1mm}
    \begin{tikzpicture}%
\node[draw=black] {%
	\begin{tabular}{llllll}
	\entrysmall{cyan!50!blue}{QL}&
	\entrysmall{green!60!black}{HRM}&
	\entrysmall{orange}{CRM}&
	\entrysmall{cyan}{QL+RS}&
	\entrysmall{green!80!black}{HRM+RS}&
	\entrysmall{yellow!80!red}{CRM+RS}    	
	\end{tabular}};%
\end{tikzpicture}%
    \caption{Results on the Minecraft-like gridworld.}
    \label{fig:discrete}
\end{figure}
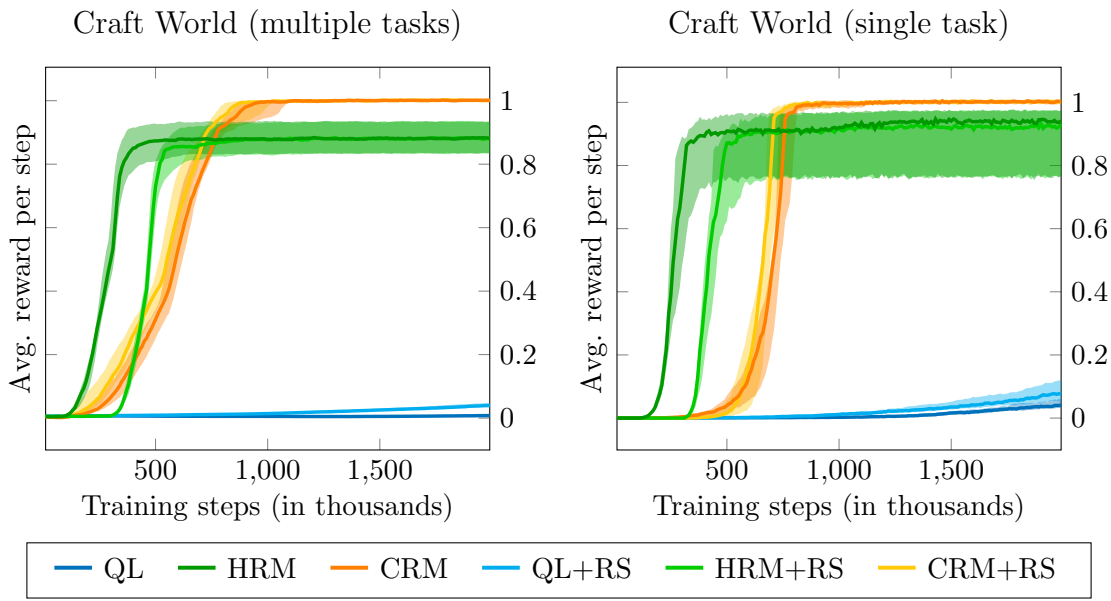

Figure~\ref{fig:discrete} shows performance over 10 randomly generated maps, with 6 trials per map. We report results for multitask and single task learning. From these results, we can draw similar conclusions as for the office world. However, notice that the gap between methods that exploit the structure of the reward machine (CRM and HRM) and methods that do not (QL) is much larger in the Minecraft domain. The reason is that the Minecraft domain is more complex and the reward sparser -- making the use of Q-learning alone hopeless.

Overall, the results on these two domains show that exploiting the RM structure can greatly  increase the performance in tabular domains for single task and multitask learning. CRM seems to be the best compromise between performance and convergence guarantees. HRM can find good policies quickly, but it often converges to suboptimal solutions. Finally, note that reward shaping helped CRM and Q-learning but did not help HRM.

\subsection{Results on Continuous State Domains}
\label{sec:res-water}

We tested our approaches in a continuous state space problem called the \emph{water world} \shortcite{sidor2016reinforcement,karpathy2015waterworld}. This environment consists of a two dimensional box with balls of different colors in it (see Figure~\ref{fig:water_dom} for an example). Each ball moves in one direction at a constant speed and bounces when it collides with the box's edges. The agent, represented by a white ball, can increase its velocity in any of the four cardinal directions. As the ball positions and velocities are real numbers, this domain cannot be tackled using tabular RL. 

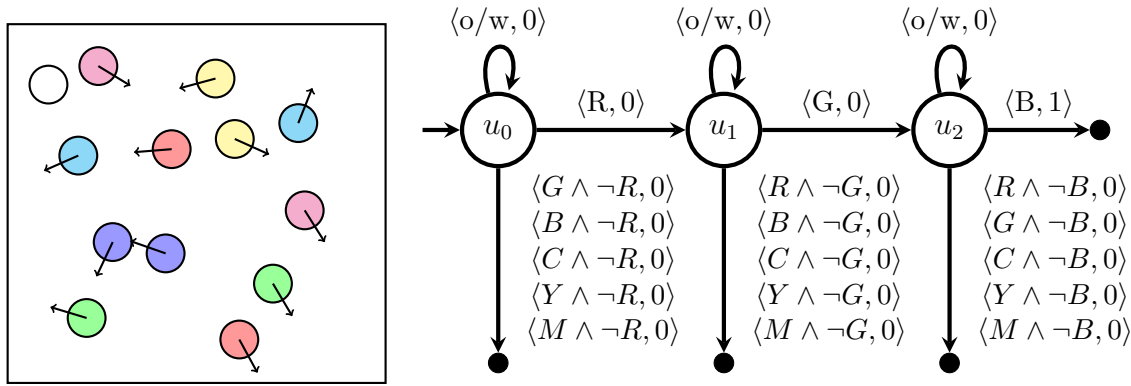
\begin{figure}
    \centering
    \begin{subfigure}{0.34\columnwidth}
        \centering
        \resizebox{\textwidth}{!}{\newcommand{\ballradius}{0.5cm}

\begin{tikzpicture}
\draw[ultra thick] (0,0) rectangle (10,9.5);
\filldraw[fill=red!40!white, draw=black, ultra thick] (6.13, 1.15) circle (\ballradius);
\filldraw[fill=red!40!white, draw=black, ultra thick] (4.34, 6.19) circle (\ballradius);
\filldraw[fill=blue!40!white, draw=black, ultra thick] (4.18, 3.43) circle (\ballradius);
\filldraw[fill=blue!40!white, draw=black, ultra thick] (2.78, 3.73) circle (\ballradius);
\filldraw[fill=green!40!white, draw=black, ultra thick] (2.09, 1.72) circle (\ballradius);
\filldraw[fill=green!40!white, draw=black, ultra thick] (7.02, 2.63) circle (\ballradius);
\filldraw[fill=cyan!40!white, draw=black, ultra thick] (1.87, 6.02) circle (\ballradius);
\filldraw[fill=cyan!40!white, draw=black, ultra thick] (7.69, 6.88) circle (\ballradius);
\filldraw[fill=magenta!40!white, draw=black, ultra thick] (7.86, 4.57) circle (\ballradius);
\filldraw[fill=magenta!40!white, draw=black, ultra thick] (2.40, 8.39) circle (\ballradius);
\filldraw[fill=yellow!40!white, draw=black, ultra thick] (6.01, 6.46) circle (\ballradius);
\filldraw[fill=yellow!40!white, draw=black, ultra thick] (5.50, 8.06) circle (\ballradius);
\filldraw[fill=white!40!white, draw=black, ultra thick] (1.08, 7.90) circle (\ballradius);
\draw [->, ultra thick] (6.13, 1.15) -- (6.61, 0.27);
\draw [->, ultra thick] (4.34, 6.19) -- (3.34, 6.11);
\draw [->, ultra thick] (4.18, 3.43) -- (3.24, 3.76);
\draw [->, ultra thick] (2.78, 3.73) -- (2.35, 2.82);
\draw [->, ultra thick] (2.09, 1.72) -- (1.13, 2.01);
\draw [->, ultra thick] (7.02, 2.63) -- (7.53, 1.77);
\draw [->, ultra thick] (1.87, 6.02) -- (0.96, 5.62);
\draw [->, ultra thick] (7.69, 6.88) -- (8.06, 7.81);
\draw [->, ultra thick] (7.86, 4.57) -- (8.39, 3.73);
\draw [->, ultra thick] (2.40, 8.39) -- (3.25, 7.87);
\draw [->, ultra thick] (6.01, 6.46) -- (6.92, 6.05);
\draw [->, ultra thick] (5.50, 8.06) -- (4.54, 7.80);
\end{tikzpicture}}
    \end{subfigure}
    \begin{subfigure}{0.65\columnwidth}
        \centering
        \begin{tikzpicture}[node distance=2cm,on grid,every initial by arrow/.style={ultra thick,->, >=stealth}]
  \node[ultra thick,state,initial,initial text=] (q_0) at (0,0) {\large$u_0$};
  \node[ultra thick,state]         (q_1) at (3,0) {\large$u_1$};
  \node[ultra thick,state]         (q_2) at (6,0) {\large$u_2$};
  \node[circle,draw=black,minimum size=0.26cm,inner sep=0pt,fill=black] (t1) at (0,-3.1)  {};
  \node[circle,draw=black,minimum size=0.26cm,inner sep=0pt,fill=black] (t2) at (3,-3.1)  {};
  \node[circle,draw=black,minimum size=0.26cm,inner sep=0pt,fill=black] (t3) at (6,-3.1)  {};
  \node[circle,draw=black,minimum size=0.26cm,inner sep=0pt,fill=black] (t4) at (8,0)  {};
  \path[ultra thick,->, >=stealth] (q_0) edge node [above] {$\tuple{\text{R},0}$} (q_1);
  \path[ultra thick,->, >=stealth] (q_1) edge node [above] {$\tuple{\text{G},0}$} (q_2);
  \path[ultra thick,->, >=stealth] (q_2) edge node [above] {$\tuple{\text{B},1}$} (t4);

  \path[ultra thick,->, >=stealth] (q_0) edge node [right] 
  {\begin{tabular}{c}
    $\tuple{G \wedge \neg R,0}$\\
    $\tuple{B \wedge \neg R,0}$\\
    $\tuple{C \wedge \neg R,0}$\\
    $\tuple{Y \wedge \neg R,0}$\\
    $\tuple{M \wedge \neg R,0}$
  \end{tabular}} (t1);
  \path[ultra thick,->, >=stealth] (q_1) edge node [right]  
  {\begin{tabular}{c}
    $\tuple{R \wedge \neg G,0}$\\
    $\tuple{B \wedge \neg G,0}$\\
    $\tuple{C \wedge \neg G,0}$\\
    $\tuple{Y \wedge \neg G,0}$\\
    $\tuple{M \wedge \neg G,0}$
  \end{tabular}} (t2);
  \path[ultra thick,->, >=stealth] (q_2) edge node [right]  
  {\begin{tabular}{c}
    $\tuple{R \wedge \neg B,0}$\\
    $\tuple{G \wedge \neg B,0}$\\
    $\tuple{C \wedge \neg B,0}$\\
    $\tuple{Y \wedge \neg B,0}$\\
    $\tuple{M \wedge \neg B,0}$
  \end{tabular}} (t3);

  \path[ultra thick,->, >=stealth] (q_0) edge [loop above] node {$\tuple{\text{o/w}, 0}$} ();
  \path[ultra thick,->, >=stealth] (q_1) edge [loop above] node {$\tuple{\text{o/w}, 0}$} ();
  \path[ultra thick,->, >=stealth] (q_2) edge [loop above] node {$\tuple{\text{o/w}, 0}$} ();
\end{tikzpicture}
    \end{subfigure}
    \caption{Water world domain and an examplary task. The events R, G, B, C, Y, and M represent touching a red ball, green ball, blue ball, cyan ball, yellow ball, and magenta ball, respectively. The label o/w stands for \textit{otherwise}. The RM encodes task 10 from Table~\ref{tab:wwt}.}
    \label{fig:water_dom}
\end{figure}

We defined a set of 10 tasks for the water world over the events of touching a ball of a certain color. For instance, one simple task consists of touching a \emph{cyan ball} after a \emph{blue ball}. Other more complicated tasks include touching a sequence of balls, such as \emph{red}, \emph{green}, and \emph{blue}, in a strict order, such that the agent fails if it touches a ball of a different color than the next one in the sequence. The complete list of tasks can be found in Table~\ref{tab:wwt}.

\begin{table}
\caption{Tasks for the water world. By ``(color1 then color2)'' we mean the task of touching a ball of color1 and then touching a ball of color2. A task like ``(color1 then color2 then color3)'' is similar but with three types of balls to touch. By ``(subtask1) and (subtask2)'' we mean that the agent must complete the tasks described by subtask1 and subtask2, but in any order. By ``(color1 strict-then color2)'' we mean the task which is like ``(color1 then color2)'' but where the agent is not allowed to touch balls of other colors during execution.}\label{tab:wwt}
\centering
\begin{tabular}{rl}
    \# & Description \\
    \hline
    1& (red then green)\\
    2& (blue then cyan)\\
    3& (magenta then yellow)\\
    4& (red then green) and (blue then cyan)\\
    5& (blue then cyan) and (magenta then yellow)\\
    6& (red then green) and (magenta then yellow)\\
    7& (red then green) and (blue then cyan) and (magenta then yellow)\\
    8& (red then green then blue) and (cyan then magenta then yellow)\\
    9& (cyan strict-then magenta strict-then yellow)\\
    10& (red strict-then green strict-then blue)
\end{tabular}
\end{table}

In these experiments, we replaced Q-learning with Double DQN. Concretely, we evaluated double DQN over the cross-product (DDQN), DDQN with counterfactual experiences (CRM), our hierarchical RL method (HRM), and their variants using reward shaping (DDQN+RS, CRM+RS, and HRM+RS). For all approaches other than HRM, we used a feed-forward network with 3 hidden layers and 1024 relu units per layer. We trained the networks using a learning rate of $10^{-5}$. On every step, we updated the Q-functions using $32n$ sampled experiences from a replay buffer of size $50000n$, where $n=1$ for DDQN and $n=|U|$ for CRM. The target networks were updated every $100$ training steps and the discount factor $\gamma$ was $0.9$. For HRM, we use the same feed-forward network and hyperparameters to train the option's policies (although $n=|\mathcal{A}|$ in this case). The high-level policy was learned using DDQN but, since the high-level decision problem is simpler, we used a smaller network (2 layers with 256 relu units) and a larger learning rate ($10^{-3}$). Our DDQN implementation was based on the code from OpenAI Baselines \shortcite{baselines}. 

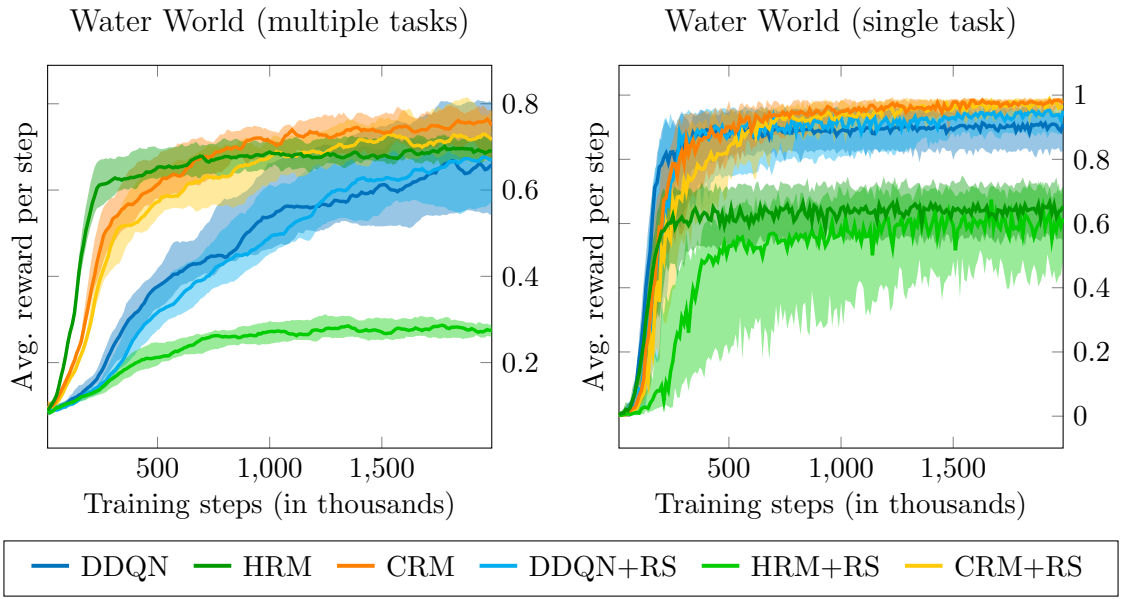
\begin{figure}
    \centering
    \begin{subfigure}{0.49\columnwidth}
        \centering
        \begin{tikzpicture}%

  \begin{axis}[%
    scale only axis,
    xlabel={Training steps (in thousands)},
    ylabel={Avg. reward per step},
    title style={font=\large},
    title={Water World (multiple tasks)},
    every axis x label/.style={at={(current axis.south)},above=-11mm},
    every axis y label/.append style={at={(current axis.west)},above=0mm},
    yticklabel pos=right,
    scaled x ticks = false,
    enlarge x limits=false,
    scale= 0.7
  ]

    \percentiles{data_V3/water-ql.txt}{0}{1}{2}{3}{cyan!50!blue}{0.4}

    \percentiles{data_V3/water-ql-rs.txt}{0}{1}{2}{3}{cyan}{0.4}

    \percentiles{data_V3/water-crm-rs.txt}{0}{1}{2}{3}{yellow!80!red}{0.4}

    \percentiles{data_V3/water-crm.txt}{0}{1}{2}{3}{orange}{0.4}

    \percentiles{data_V3/water-hrm-rs.txt}{0}{1}{2}{3}{green!80!black}{0.4}

    \percentiles{data_V3/water-hrm.txt}{0}{1}{2}{3}{green!60!black}{0.4}

  \end{axis}
\end{tikzpicture}%
    \end{subfigure}
    \begin{subfigure}{0.49\columnwidth}
        \centering
        \begin{tikzpicture}%

  \begin{axis}[%
    scale only axis,
    xlabel={Training steps (in thousands)},
    ylabel={Avg. reward per step},
    title style={font=\large},
    title={Water World (single task)},
    every axis x label/.style={at={(current axis.south)},above=-11mm},
    every axis y label/.append style={at={(current axis.west)},above=0mm},
    yticklabel pos=right,
    scaled x ticks = false,
    enlarge x limits=false,
    scale= 0.7
  ]

    \percentiles{data_V3/water-single-ql.txt}{0}{1}{2}{3}{cyan!50!blue}{0.4}

    \percentiles{data_V3/water-single-ql-rs.txt}{0}{1}{2}{3}{cyan}{0.4}

    \percentiles{data_V3/water-single-crm-rs.txt}{0}{1}{2}{3}{yellow!80!red}{0.4}

    \percentiles{data_V3/water-single-crm.txt}{0}{1}{2}{3}{orange}{0.4}

    \percentiles{data_V3/water-single-hrm-rs.txt}{0}{1}{2}{3}{green!80!black}{0.4}

    \percentiles{data_V3/water-single-hrm.txt}{0}{1}{2}{3}{green!60!black}{0.4}

  \end{axis}
\end{tikzpicture}%
    \end{subfigure}

    \vspace{1mm}
\begin{tikzpicture}%
\node[draw=black] {%
	\begin{tabular}{l@{\hspace{3mm}}l@{\hspace{3mm}}l@{\hspace{3mm}}l@{\hspace{3mm}}l@{\hspace{3mm}}l}
	\entrysmall{cyan!50!blue}{DDQN}&
	\entrysmall{green!60!black}{HRM}&
	\entrysmall{orange}{CRM}&
	\entrysmall{cyan}{DDQN+RS}&
	\entrysmall{green!80!black}{HRM+RS}&
	\entrysmall{yellow!80!red}{CRM+RS}    	
	\end{tabular}};%
\end{tikzpicture}%
    \caption{Results in the water world domain.}
    \label{fig:continuous}
\end{figure}

\begin{figure}
    \centering
    \begin{subfigure}{0.49\columnwidth}
        \centering
        \begin{tikzpicture}%

  \begin{axis}[%
    scale only axis,
    xlabel={Training steps (in thousands)},
    ylabel={Avg. reward per step},
    title style={font=\large},
    title={Water World (multiple tasks)},
    every axis x label/.style={at={(current axis.south)},above=-11mm},
    every axis y label/.append style={at={(current axis.west)},above=0mm},
    yticklabel pos=right,
    scaled x ticks = false,
    enlarge x limits=false,
    scale= 0.7
  ]

    \percentiles{data_V3/water-qrm.txt}{0}{1}{2}{3}{violet!70!white}{0.4}

    \percentiles{data_V3/water-crm.txt}{0}{1}{2}{3}{orange}{0.4}

    \percentiles{data_V3/water-crm1.txt}{0}{1}{2}{3}{teal}{0.4}

    \percentiles{data_V3/water-crm2.txt}{0}{1}{2}{3}{teal!40!lime}{0.4}

    \percentiles{data_V3/water-crm3.txt}{0}{1}{2}{3}{orange!40!cyan}{0.4}
  \end{axis}
\end{tikzpicture}%
    \end{subfigure}
    \begin{subfigure}{0.49\columnwidth}
        \centering
        \begin{tikzpicture}%
\node[draw=black] {%
	\begin{tabular}{l}
	\textbf{Legend:}\\
	\entry{violet!70!white}{QRM}\\
	\entry{orange}{CRM (3L/1024N)}\\
	\entry{teal}{CRM (3L/512N)}\\
	\entry{teal!40!lime}{CRM (3L/256N)}\\
	\entry{orange!40!cyan}{CRM (6L/64N)}
	\end{tabular}};%
\end{tikzpicture}%
    \end{subfigure}
    \caption{Results in the water world domain.}
    \label{fig:qrms}
\end{figure}
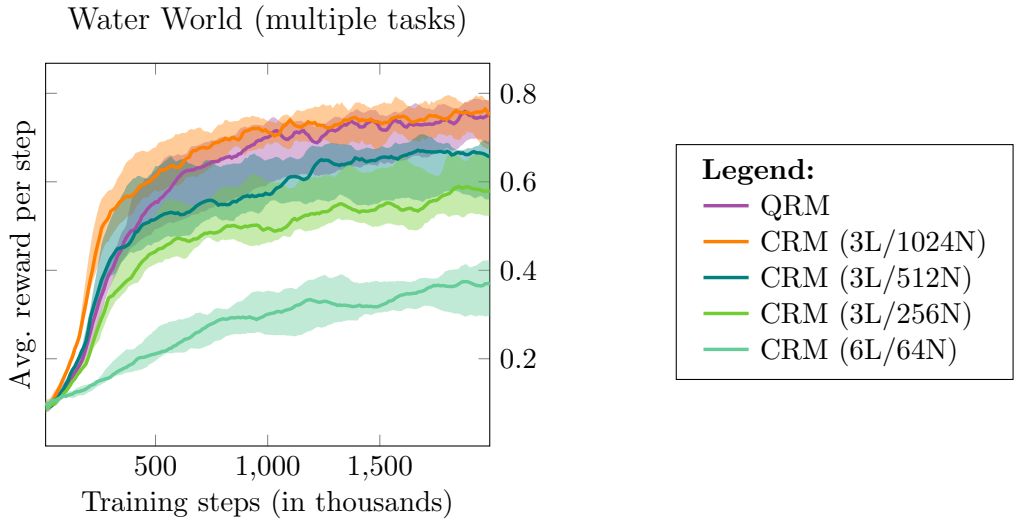

We ran experiments in multitask and single task settings. In the multitask setting, all the approaches learned one policy (i.e., network) to solve all the tasks. Figure~\ref{fig:continuous} shows the results on 10 randomly generated water world maps, with 2 trials per map. We normalized the average reward per step using the run that got the highest average reward across all the approaches. In the multitask experiments, CRM performs the best. HRM initially learns faster than CRM but converged to suboptimal policies and adding reward shaping decreased the performance of all the approaches. In the single task experiment, we evaluated the performance of all the approaches when trying to solve task 10 only (see Table~\ref{tab:wwt}). In this case, CRM also converged to better policies.

Finally, we compare the performance of CRM and QRM. As discussed in Section~\ref{sec:crm}, CRM and QRM are not equivalent when using function approximation. The most notable difference is that CRM uses a large network to learn a single policy $\pi(a|s,u)$ for all RM states whereas QRM uses a set of small networks, one to learn a policy $\pi_u(a|s)$ for each RM state $u\in U$. The results in Figure~\ref{fig:qrms} show that CRM and QRM have comparable performance on the water world. 
However, to do so CRM requires using a larger network. In this experiment, QRM is learning networks of 6 layers with 64 relu units (6L/64N) -- which are the same size as the networks used by \citeA{icml2018rms}. To get to the same performance as QRM, CRM has to use a network of 3 layers with 1024 relu units per layer. 
We leave it to future work to further investigate the trade-offs of using multiple small networks to solve a task (as in QRM) versus one network per task (as in CRM).
That said, CRM has the added advantage of being trivial to implement. In fact, we do not have results for QRM in the next section because it is unclear how to integrate QRM with DDPG.

\subsection{Results on Continuous Control Tasks}
\label{sec:res-cheetah}

\begin{figure}
    \centering
    \input{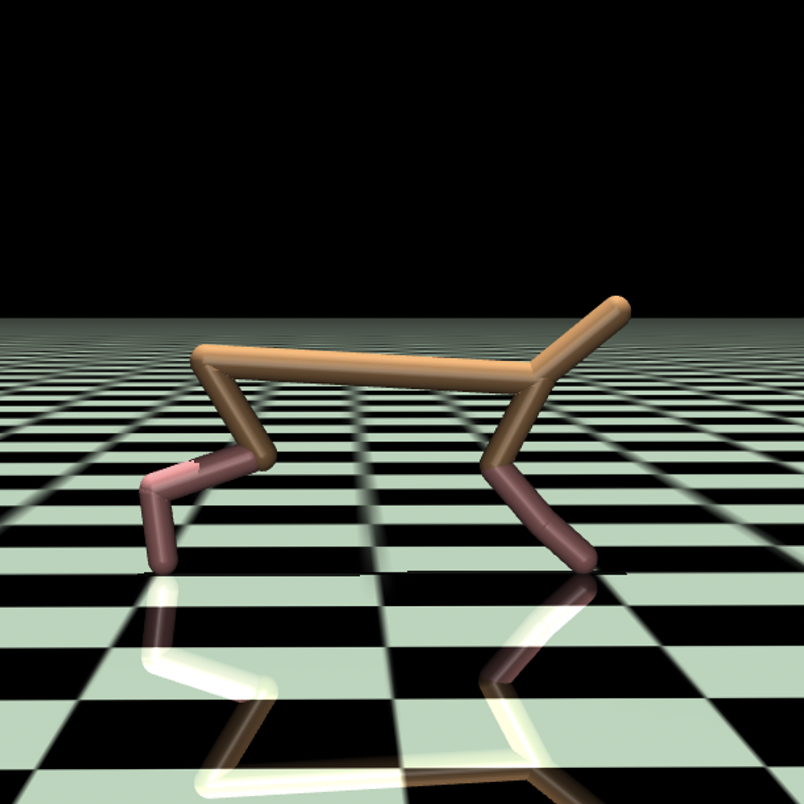}
    \caption{HalfCheetah-v3 domain. The first task consists of going from A to B and back as many times as possible. The second task consists of reaching F as soon as possible. }
    \label{fig:cheetah_dom}
\end{figure}

Our final set of experiments considers the case where the action space is continuous. We ran experiments on the HalfCheetah-v3 environment \shortcite{1606.01540}, shown in Figure~\ref{fig:cheetah_dom}. In this environment, the agent is a cheetah-like robot. This agent has 6 joints that it must learn to control in order to move forward or backwards. At each time step, the agent chooses how much force to apply to each joint, making the action space infinite. The state space is also continuous, including the location and velocities of each joint. 

We evaluated the performance of our approaches in two tasks (independently). All the approaches use DDPG as the underlying off-policy learning approach. In the case of HRL, the option policies are learned using DDPG but the high-level policy uses DQN. All the approaches use a feed-forward network with 2 layers and 256 relu units per layer. The batch size was $100n$ (where $n=1$ in DDPG, $n=|U|$ in CRM, and $n=|\mathcal{A}|$ in HRM) and the rest of the hyperparameters were set to their default values \shortcite{baselines}.

Figure~\ref{fig:cheetah}-left shows average results over 20 runs for the first task. This task consists of moving back and forth from point A to point B (shown in Figure~\ref{fig:cheetah_dom}) as many times as possible given a time limit of 1000 steps. The corresponding reward machine gives a reward of 1000 every time the agent completes a lap. The agent also receives a small \textit{control penalization} on every step -- which is a standard penalization that discourages the agent from applying large forces into its joints (it is a quadratic penalization over the agent's actions).

We note that this is a continuing task, and is also an example of a reward machine with loops and no terminal states. The inclusion of the control penalty means that the reward received depends on the current state, not just the reward machine transition taken. This task -- as well as the second used in this section -- is therefore not specified as a simple RM.

The results show that CRM largely outperforms the other approaches, completing 9 laps on average by the end of learning. HRM also performs well, completing around 6 laps per episode by the end of learning. As before, the problem with HRM is that it optimizes for reaching the next subgoal (A or B) without considering what has to be done next. While CRM learns to slow down before reaching A or B so it can quickly jump back afterwards, HRM learns to run full speed until reaching A or B and, as a result, keeps advancing a few steps before being able to slow down and start moving in the opposite direction.

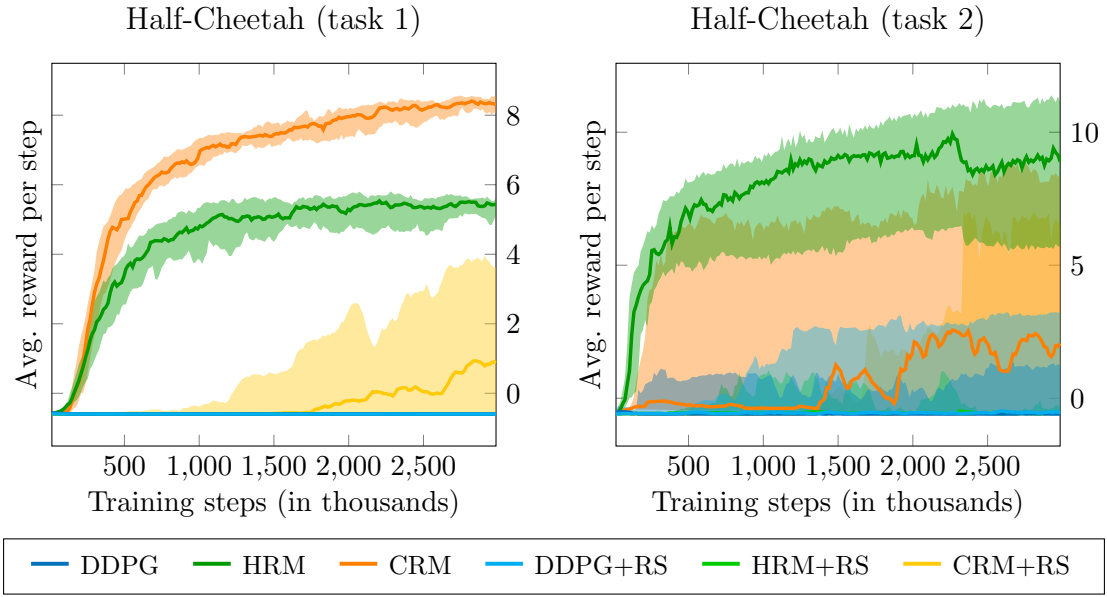
\begin{figure}
    \centering
    \begin{subfigure}{0.49\columnwidth}
        \centering
        \begin{tikzpicture}%

  \begin{axis}[%
    scale only axis,
    xlabel={Training steps (in thousands)},
    ylabel={Avg. reward per step},
    title style={font=\large},
    title={Half-Cheetah (task 1)},
    every axis x label/.style={at={(current axis.south)},above=-11mm},
    every axis y label/.append style={at={(current axis.west)},above=0mm},
    yticklabel pos=right,
    scaled x ticks = false,
    enlarge x limits=false,
    scale= 0.7
  ]

    \percentiles{data_V3/cheetah-M1-crm-rs.txt}{0}{1}{2}{3}{yellow!80!red}{0.4}

    \percentiles{data_V3/cheetah-M1-crm.txt}{0}{1}{2}{3}{orange}{0.4}

    \percentiles{data_V3/cheetah-M1-hrm-rs.txt}{0}{1}{2}{3}{green!80!black}{0.4}

    \percentiles{data_V3/cheetah-M1-hrm.txt}{0}{1}{2}{3}{green!60!black}{0.4}

    \percentiles{data_V3/cheetah-M1-ql.txt}{0}{1}{2}{3}{cyan!50!blue}{0.4}

    \percentiles{data_V3/cheetah-M1-ql-rs.txt}{0}{1}{2}{3}{cyan}{0.4}

  \end{axis}
\end{tikzpicture}%
    \end{subfigure}
    \begin{subfigure}{0.49\columnwidth}
        \centering
        \begin{tikzpicture}%

  \begin{axis}[%
    scale only axis,
    xlabel={Training steps (in thousands)},
    ylabel={Avg. reward per step},
    title style={font=\large},
    title={Half-Cheetah (task 2)},
    every axis x label/.style={at={(current axis.south)},above=-11mm},
    every axis y label/.append style={at={(current axis.west)},above=0mm},
    yticklabel pos=right,
    scaled x ticks = false,
    enlarge x limits=false,
    scale= 0.7
  ]

    \percentiles{data_V3/cheetah-M2-crm-rs.txt}{0}{1}{2}{3}{yellow!80!red}{0.4}

    \percentiles{data_V3/cheetah-M2-crm.txt}{0}{1}{2}{3}{orange}{0.4}

    \percentiles{data_V3/cheetah-M2-hrm-rs.txt}{0}{1}{2}{3}{green!80!black}{0.4}

    \percentiles{data_V3/cheetah-M2-hrm.txt}{0}{1}{2}{3}{green!60!black}{0.4}

    \percentiles{data_V3/cheetah-M2-ql.txt}{0}{1}{2}{3}{cyan!50!blue}{0.4}

    \percentiles{data_V3/cheetah-M2-ql-rs.txt}{0}{1}{2}{3}{cyan}{0.4}

  \end{axis}
\end{tikzpicture}%
    \end{subfigure}
    
    \vspace{1mm}
\begin{tikzpicture}%
\node[draw=black] {%
	\begin{tabular}{llllll}
	\entrysmall{cyan!50!blue}{\small DDPG}&
	\entrysmall{green!60!black}{\small HRM}&
	\entrysmall{orange}{\small CRM}&
	\entrysmall{cyan}{\small DDPG+RS}&
	\entrysmall{green!80!black}{\small HRM+RS}&
	\entrysmall{yellow!80!red}{\small CRM+RS}    	
	\end{tabular}};%
\end{tikzpicture}%
    \caption{Results in the HalfCheetah-v3 domain.}
    \label{fig:cheetah}
\end{figure}

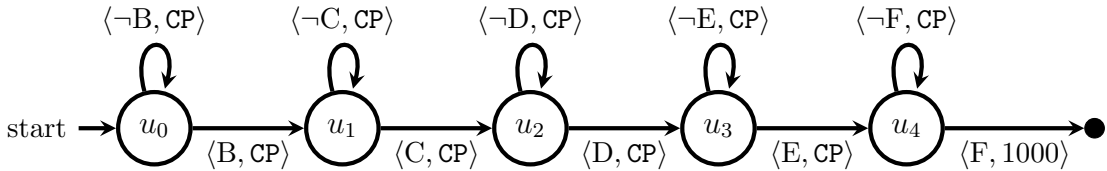
\begin{figure}[t]
    \centering
    \begin{tikzpicture}[node distance=2cm,on grid,every initial by arrow/.style={ultra thick,->, >=stealth}]
  \node[ultra thick,state,initial] (q_0) at (0,0) {\large$u_0$};
  \node[ultra thick,state]         (q_1) at (2.5,0) {\large$u_1$};
  \node[ultra thick,state]         (q_2) at (5,0) {\large$u_2$};
  \node[ultra thick,state]         (q_3) at (7.5,0) {\large$u_3$};
  \node[ultra thick,state]         (q_4) at (10,0) {\large$u_4$};
  \node[circle,draw=black,minimum size=0.26cm,inner sep=0pt,fill=black] (t1) at (12.5,0)  {};
  \path[ultra thick,->, >=stealth] (q_0) edge node [below] {$\tuple{\text{B},\texttt{CP}}$} (q_1);
  \path[ultra thick,->, >=stealth] (q_1) edge node [below] {$\tuple{\text{C},\texttt{CP}}$} (q_2);
  \path[ultra thick,->, >=stealth] (q_2) edge node [below] {$\tuple{\text{D},\texttt{CP}}$} (q_3);
  \path[ultra thick,->, >=stealth] (q_3) edge node [below] {$\tuple{\text{E},\texttt{CP}}$} (q_4);
  \path[ultra thick,->, >=stealth] (q_4) edge node [below] {$\tuple{\text{F},1000}$} (t1);

  \path[ultra thick,->, >=stealth] (q_0) edge [loop above] node {$\tuple{\neg \text{B}, \texttt{CP}}$} ();
  \path[ultra thick,->, >=stealth] (q_1) edge [loop above] node {$\tuple{\neg \text{C}, \texttt{CP}}$} ();
  \path[ultra thick,->, >=stealth] (q_2) edge [loop above] node {$\tuple{\neg \text{D}, \texttt{CP}}$} ();
  \path[ultra thick,->, >=stealth] (q_3) edge [loop above] node {$\tuple{\neg \text{E}, \texttt{CP}}$} ();
  \path[ultra thick,->, >=stealth] (q_4) edge [loop above] node {$\tuple{\neg \text{F}, \texttt{CP}}$} ();
\end{tikzpicture}
    \caption{Reward machine for the second task on the HalfCheetah-v3 domain. \texttt{CP} represents the control penalty usually used in this domain. }
    \label{fig:cheetah_task2}
\end{figure}

Note that the strength of CRM relies on being able to effectively share experience among different RM states. Being able to achieve one task while trying to achieve a different one (for example, reaching point A while trying to get to point B) allows CRM to learn efficiently. If that sort of behaviour cannot happen, then CRM would not help much since the shared experience would not provide new insights into how to solve the problem. In contrast, HRM performs the best in tasks where its problem decomposition preserves optimal policies. To provide a complete view of the strengths and weaknesses of these approaches, we designed a second task for HalfCheetah-v3 that is ideal for HRM and difficult for CRM.

This second task consists of reaching F (shown in Figure~\ref{fig:cheetah_dom}) and it is represented by the RM from Figure~\ref{fig:cheetah_task2}. This RM is a chain of 6 states that advances as the agent reaches B, C, D, E, and F. The agent gets a reward of 1000 when it reaches F and the control penalization (\texttt{CP}) otherwise. This is a challenging task for DDPG because the reward is very sparse. Adding CRM can make the reward less sparse if the probability that some RM state gets a positive reward increases by adding counterfactual experiences. This was the case in the first task, but it is not in this second task. On the other hand, HRM is a good fit for this problem for two reasons. First, the reward is less sparse for HRM since the agent gets some reward signal every time B, C, D, E, or F are reached. And second, the policy that results from composing optimal option policies is close to globally optimal.

The average results over 20 runs for the second task are shown in Figure~\ref{fig:cheetah}-right. As expected, HRM is the approach that performs best, as it is able to reach F in less than 90 steps. Interestingly, adding CRM to DDPG did help a bit, allowing the agent to reach F in 120 steps in some runs, but its performance was unreliable.

\subsection{Runtime Comparison}

The strong sample efficiency of CRM and HRM comes with a caveat, though. These two approaches are more computationally expensive than the cross-product baseline. In fact, Table~\ref{tab:runtimes} shows the average runtime for each approach in our domains. Generally speaking, the most computationally expensive approach is HRM followed by CRM (and there is almost no difference between adding or not reward shaping). The results on the Office and Craft domains were computed using one core on an Intel(R) Xeon(R) Gold 6148 processor. The results on the Water and Half-Cheetah domains were computed using one Tesla P100 GPU. 

These results are not surprising. In tabular RL, CRM performs one Q-update per RM state and HRM performs one Q-update per RM edge in every iteration. In contrast, the cross-product baseline performs only one Q-update per iteration. In the case of deep RL, CRM and HRM have the overhead of creating the counterfactual experiences, adding them to the replay buffer, and updating the network's weights using larger mini-batches. Note that a reason CRM and HRM have more overhead in the multitask setting is that counterfactual experiences are generated for every state in every task, not just the current task.

The good news is that the Q-updates performed by CRM and HRM can be done in parallel to reduce their runtimes. In fact, our implementations of CRM and HRM with deep RL exploit some degree of parallelism as the computation of gradients using larger mini-batches is parallelized by TensorFlow.  Therefore, the gap between the cross-product baseline, CRM, and HRM, can be reduced by using a stronger GPU. 

Finally, note that the advantages of using CRM or HRM go beyond improving sample efficiency. They decrease the sparsity of the reward signal and, as such, CRM and HRM might find good policies in problems where the cross-product baseline would have a hard time just finding any reward. This is the case for task 1 in the Half-Cheetah environment and its consequence is that the cross-product baseline is unlikely to solve such a problem regardless of how long it runs. In fact, we ran DDPG for 30 millions steps (over 3 days of computation) and was still unable to find a better-than-random policy for task 1.

\begin{table}
\caption{Runtime comparison. We use CP to denote the cross-product baseline (which is either Q-learning, DDQN, or DDPG, depending on the domain). The setup can be ST (single task), MT (multiple tasks), T1 (task 1), or T2 (task 2). We report average runtime and its standard deviation across all random seeds and maps per domain and approach.}\label{tab:runtimes}
\centering
{\small
\begin{tabular}{ccrrrrrr}
    \toprule
    Domain & Setup & \multicolumn{1}{c}{CP} & \multicolumn{1}{c}{CP+RS} & \multicolumn{1}{c}{HRM} & \multicolumn{1}{c}{HRM+RS} & \multicolumn{1}{c}{CRM} & \multicolumn{1}{c}{CRM+RS} \\\midrule
    Office World & ST & $2.5 \pm 0.1$ & $3.3 \pm 0.1$ & $16.1 \pm 0.5$ & $17.7 \pm 0.4$ & $11.9 \pm 0.3$ & $12.1 \pm 0.4$ \\
    (in seconds) & MT & $2.9 \pm 0.1$ & $3.7 \pm 0.1$ & $38.6 \pm 0.5$ & $40.6 \pm 0.7$ & $32.1 \pm 0.9$ & $32.3 \pm 0.9$ \\\midrule
    Craft World & ST & $0.9 \pm 0.0$ & $1.1 \pm 0.0$ & $8.8 \pm 0.2$ & $9.4 \pm 0.4$ & $6.3 \pm 0.1$ & $6.4 \pm 0.3$  \\
    (in minutes) & MT & $1.3 \pm 0.0$ & $1.6 \pm 0.0$ & $62.4 \pm 1.5$ & $64.1 \pm 2.4$ & $49.2 \pm 2.5$ & $50.3 \pm 2.2$  \\\midrule
    Water World & ST &$3.1 \pm 0.1$ & $3.0 \pm 0.1$ & $3.8 \pm 0.2$ & $3.7 \pm 0.2$ & $3.4 \pm 0.2$ & $3.5 \pm 0.2$  \\
    (in hours) & MT & $3.1 \pm 0.2$ & $3.1 \pm 0.1$ & $37.1 \pm 3.0$ & $36.4 \pm 2.8$ & $23.4 \pm 2.8$ & $21.1 \pm 0.4$ \\\midrule
    Half-Cheetah & T1 & $7.7 \pm 0.1$ & $6.9 \pm 0.5$ & $5.2 \pm 0.2$ & $4.8 \pm 0.5$ & $6.4 \pm 0.2$ & $6.2 \pm 0.8$ \\
    (in hours) & T2 & $7.1 \pm 0.7$ & $6.8 \pm 0.5$ & $5.7 \pm 0.2$ & $6.3 \pm 0.5$ & $7.4 \pm 0.6$ & $6.9 \pm 0.4$ \\
    \bottomrule
\end{tabular}}
\end{table}

\subsection{Code}
Our code is available at \url{github.com/RodrigoToroIcarte/reward_machines}, including our environments, raw results, and implementations of the cross-product baseline, automated reward shaping, CRM, and HRM using tabular Q-learning, DDQN, and DDPG. For the experiments with QRM, we use the following implementation: \url{bitbucket.org/RToroIcarte/qrm}. Our methods are fully integrated with the OpenAI Gym API \cite{1606.01540}.

\section{Related Work}
\label{sec:related_work}
In this section, we discuss existing works on the topic of reward machines and how this paper fits within that body of literature. We then discuss how reward machines relate more generally with approaches for reward specifications and knowledge exploitation in RL.

\subsection{Reward Machine Research}

We originally proposed reward machines in an ICML publication \cite{icml2018rms}. At that time, there had also been work on using Linear Temporal Logic (LTL) or related languages to reward agents in MDPs \shortcite<e.g.,>{bacchus1996,lacerda2014optimal,lacerda2015optimal,cam-che-san-mci-socs17,brafman2017specifying} and RL \shortcite<e.g.,>{AksarayJKSB16,littman2017environment,li2016reinforcement,LiMB18,hasanbeig2018logically}. A popular approach was to translate the LTL specification into a finite state machine, reward the agent when the machine hits an accepting state, and learn policies using the cross-product baseline. In contrast, prior to our reward machine work, we had proposed a novel approach to RL with LTL rewards, called LPOPL, that did not exploit automata but rather exploited the structure of LTL natively to learn policies faster than existing methods \shortcite{aamas2018lpopl}. LPOPL is QRM's predecessor as it relies on the same learning principle: It decomposes the LTL tasks into many subtasks and learns policies for them in parallel via off-policy learning.

The main contributions of our ICML paper were to introduce RMs and QRM \cite{icml2018rms}. QRM generalizes LPOPL and also outperforms the cross-product baseline and Hierarchical RL. Since then, RMs have been used for solving problems in planning \shortcite{illanesYTM2019symbolic,IllanesYIM20}, robotics \shortcite{shah2020planning,shah2020interactive,defazio2021learning,camacho2020disentangled,camacho2021reward}, multi-agent systems \shortcite{neary2020reward}, lifelong RL \shortcite{zheng2021lifelong}, and partial observability \shortcite{tor-etal-neurips19}. \shortciteA{DeGiacomo2020transducers} also considered both Mealy and Moore versions of RMs, though theirs only output numbers (like our simple RMs) instead of reward functions. Finally, there has been prominent work on how to learn RMs from experience \shortcite<e.g.,>{tor-etal-neurips19,toro2019rldm,icarte2021learning,xu2020joint,xu2020active,furelos2020induction,furelos2020inductionAAAI,rens2020learning,hasanbeig2019deepsynth,velasquez2021learning}

Since our previous work, we have gained practical experience and new theoretical insights about reward machines -- which were reflected in this paper. In particular, we provided a cleaner definition of reward machines and QRM. On the RM side, we changed $\delta_r$ from returning a reward function on each transition to returning a reward function on each state (i.e., changed from a Mealy to Moore formulation). As discussed in Section \ref{sec:reward-machines}, this is as expressive as before but simpler. We also added terminal states to reward machines because terminal states naturally arise in most practical applications. On the QRM side, we proposed CRM as a novel view of QRM that is simpler to understand and implement. Another improvement was the addition of HRM. In our ICML paper, hierarchical RL was a baseline. We hand-picked the set of options and learned policies following recommendations from \citeA{sutton1999between} and \shortciteA{kulkarni2016hierarchical}. The key difference between HRM and our previous HRL baseline is that HRM automatically extracts the set of options from the reward machine.  Our new experiments have demonstrated the effectiveness of this approach, and thus we consider HRM itself to be a general HRL-based method for solving MDPRMs. 

The experimental evaluation in this paper also improves upon that first described in the original reward machine work. First, we have used a better performance metric: the average reward per step instead of the normalized discounted reward (see discussion below). Second, we have added experiments on a continuous control environment. These are the first known results on continuous control for reward machines. And third, we include single-task experiments. Our ICML paper only had multitask experiments, which created the misconception that QRM only worked for multitask learning. We addressed those concerns in this paper. Finally, we reimplemented our code and made it fully compatible with OpenAI Gym. We hope this will facilitate future research on reward machines.

\subsubsection{Average Reward Per Step vs Normalized Discounted Return}
The reason to prefer \textit{average reward per step (ARPS)} over \textit{normalized discounted reward (NDR)} is that NDR is a metric that penalizes suboptimal steps exponentially and depends on the chosen value of $\gamma$. For instance, let's consider problems that gives a reward of 1 for completing a task and zero otherwise (as in most of our experiments). Imagine that an optimal policy solves a certain task in 100 steps and that HRM, which converges to suboptimal solutions, solves it in 104 steps. Then, the NDR performance of HRM will be $\gamma^4$. Meanwhile, the performance of CRM will converge to $1.0$ since CRM converges to optimal policies in tabular domains. The problem with NDR is that $\gamma^4$ can go from zero to one depending on the value of $\gamma$. If $\gamma = 0.9$, the HRM performance will be $0.66$. If $\gamma = 0.8$, it will be $0.41$. In contrast, the ARPS performance of HRM will always be $0.96$ in this experiment, regardless of the value of $\gamma$. This makes ARPS results easier to interpret than the NDR in our setting.

\subsection{Reward Specification}
There has been significant interest in using formal languages to specify tasks, constraints, and advice in reinforcement learning \shortcite<e.g.,>{li2016reinforcement,littman2017environment,toro2017rldm,toroicarteKVM2018advice,aamas2018lpopl,hasanbeig2018logically,Hasanbeig2019Reinforcement,Hasanbeig2019Certified,HasanbeigKA19Towards,HasanbeigAK19Logically,HasanbeigAK20Cautious,HasanbeigKA20Deep,li2019temporal,li2019formal,ringstrom2019constraint,quint2019formal,jothimurugan2019composable,bozkurt2019control,koroglu2019reinforcement,gaon2019reinforcement,Yuan2019Modular,shah2020planning,shah2020interactive,ghasemi2020task,de2020restraining,li2020formal,leon2020systematic,jiang2020temporal,bozkurt2021learning,hammond2021multi,luo2021temporal,araki2021logical,cai2021modular,vaezipoor2021ltl2action}. We hope to see more research in this direction in the next years. 

The use of formal languages may 
facilitate specification of reward functions for RL in complex systems. Further, formal languages typically have compositional structure that RL agents could, in principle, exploit to learn policies more efficiently. That said, creating learning methods tailored to the myriad of languages people might wish to employ -- LTL, LDL (Linear Dynamic Logic), and so on -- is labour intensive. To mitigate this, \citeA{camamacho2019ijcai} proposed specifying reward functions in the developer's language of choice and using reward machines as a normal form representation for learning. This allows us to focus our efforts on two subproblems: (i) understanding how to translate particular languages into equivalent RMs and (ii) understanding how to exploit the RM structure to learn policies faster. We have made progress toward both problems. While this paper presents methods for exploiting RM structure in learning, \citeA{camamacho2019ijcai} present a way of translating a number of popular formal languages into RMs by leveraging well-understood relationships between languages and automata, and \citeA{middleton2020icaps} developed a tool that supports a subset of these translations.

That said, we note that designing reward functions that lead to the intended behavior can be difficult. The designer may neglect to penalize some undesirable changes in the environment, resulting in policies that are optimal with respect to the specification but that have negative side effects. More generally, the agent may find unintended ways to maximize reward \shortcite{amodei2016concrete}.
\shortciteA{hadfieldmenell2017ird} have asserted that manually designed reward functions should merely be viewed as evidence relating to what the designer actually intended.
Reward machines in general do not address this problem, though they may make it easier to specify temporally extended behaviors.

Some alternatives to manually designed reward functions are \emph{demonstrations} \shortcite<e.g.,>{ng2000irl,abbeel2004apprenticeship,ziebart2008maxent,fu2018robust}, \emph{positive and negative feedback} \shortcite<e.g.,>{thomaz2006reinforcement2,knox2008tamer,macglashan2017interactive}, and \emph{trajectory preferences} \shortcite<e.g.,>{akrour2012april,christiano2017preferences}. When using demonstrations, tasks are specified using a set of expert traces. Then, inverse reinforcement learning (IRL) is used to transform traces into reward functions. Positive and negative feedback which comes from an expert who observes the agent during training is also an alternative to specifying a reward function in advance. Trajectory preferences are another form of feedback that require the expert to just compare trajectories created by the agent and say which they prefer.
Demonstrations and feedback are a useful proxy for task specifications, but unlike reward machines they do not specify the task itself. 
Future work could look at ways to combine demonstrations and feedback with RMs. For example, the user could specify the graphical structure of an RM and then use IRL to learn the state-reward function of each RM state.

\subsection{Exploiting Prior Knowledge}

Our approaches for exploiting RM structure are inspired by methods for exploiting prior knowledge in RL. In particular, prior knowledge has been used for problem decomposition \shortcite{parr1998reinforcement,dietterich2000hierarchical,mann2015approximate}, data augmentation \shortcite{andrychowicz2017hindsight,pitis2020counterfactual}, and reward shaping \shortcite{Ng1999shaping}. 

Hierarchical reinforcement learning is the most successful methodology to exploit decompositions in RL. Some foundational HRL works include \emph{H-DYNA} \cite{singh1992reinforcement}, \emph{MAXQ} \cite{dietterich2000hierarchical}, \emph{HAMs} \cite{parr1998reinforcement}, and \emph{Options} \shortcite{sutton1999between}. The role of the hierarchy is to decompose the task into a set of sub-tasks that are reusable and easier to learn. However, these methods cannot guarantee convergence to optimal policies because hierarchies constrain the policy space and, hence, might prune optimal policies. An RM can be viewed as a form of hierarchy that also defines the reward function. This allows us to define methods that can speed up learning and still guarantee convergence to optimal policies (e.g., QRM, CRM, or RS). We also proposed HRM, which automatically extracts an option-based hierarchy from an RM to solve MDPRMs faster (although it inherits the possibility of convergence to suboptimal policies from the option framework).

\citeA{singh1992reinforcement,singh1992transfer} proposed an alternative to HRL which defines tasks as sequences of sub-goals. Independent policies are trained to achieve each sub-goal, and then a gating function learns to switch from one policy to the next. The same idea was exploited by \emph{policy sketches} \cite{andreas2016sketches} but without the need for an external signal when a sub-goal is reached. In contrast, reward machines are considerably more expressive than sub-goal sequences and sketches, as they allow for interleaving, loops, and compositions of entire reward functions. Indeed, regular expressions can be captured in finite state machines.

CRM exploits a similar learning principle as Hindsight Experience Replay (HER) and Counterfactual Data Augmentation (CoDA). HER was proposed by \shortciteA{andrychowicz2017hindsight} and, like CRM, relies on relabelling experiences to learn policies faster. However, CRM relabels experiences using the RM and HER uses goal states. One advantage of CRM is that RMs can encode temporally extended behaviours that cannot be encoded using HER, such as reaching a goal state while avoiding some objects (tasks 9 and 10 in Section~\ref{sec:res-water}) or loopy behaviours (task 1 in Section~\ref{sec:res-cheetah}). That said, an advantage of HER is that it can learn policies that generalize to unseen goal states. Currently, it is unclear how to achieve a similar behaviour using RMs (i.e., to learn a policy that generalizes to unseen RMs). 

CRM is also related to CoDA \cite{pitis2020counterfactual}. CoDA is a recently proposed technique for generating counterfactual experiences in RL. This approach consists of combining two experiences to generate new counterfactual experiences by exploiting locally independent causal factors. In contrast, CRM exploits the RM to generate multiple counterfactual experiences from one single environment experience. Further exploring synergies between RMs, HER, and CoDA is a promising direction for future work.

Our method for automated reward shaping was inspired by \shortciteA{cam-che-san-mci-goalsrl18}. In that work, \citeauthor{cam-che-san-mci-goalsrl18} proposed an approach for automated reward shaping over automata in service of finding a policy for a fully specified MDP with an LTL-specified reward function. To do so, they defined a potential-based function that considered the distance between each automata state and its closest accepting state. In this paper, we extended \citeauthor{cam-che-san-mci-goalsrl18}'s approach to work over simple reward machines.

\section{Concluding Remarks}

In this paper we introduced the notion of reward machines -- a form of finite state machine that can be used to specify the reward function of an RL agent. Reward machines support the specification of arbitrary rewards, including sparse rewards and rewards for temporally extended behaviors. Reward machines expose structure in the reward function and, in doing so, can 
significantly improve the sample complexity of learning, as demonstrated in our experiments. Our methods enable us to find solutions faster and, more importantly, to solve problems that could not otherwise be solved given limited interaction with the environment. We proposed three different methodologies to exploit reward machine structure in learning: automated reward shaping, counterfactual reasoning, and decomposition methods. We discussed the convergence guarantees of these approaches in the tabular case and empirically evaluated their effectiveness in discrete and continuous domains.

Many questions remain open regarding reward machines. For instance, we know how to learn reward machines from experience \shortcite{tor-etal-neurips19,toro2019rldm,xu2020joint,xu2020active,furelos2020induction,furelos2020inductionAAAI,rens2020learning}, but all these methods assume access to a correct labelling function. How to learn RMs and a labelling function at the same time remains unknown. Similarly, it is unclear how to deal with noisy labelling functions. Here, we assume that the \emph{event detectors} that define the labelling function work perfectly, but this is unrealistic in many real world problems. If we know that a detector fails with certain probability, it is unclear what the next state of the reward machine should be. Recently, we proposed a method that learns a policy that generalize to unseen tasks defined in LTL \cite{vaezipoor2021ltl2action} but we do not know how to achieve a similar behaviour using RMs. Finally, our work focused on approaches to exploit the RM structure when using model-free RL methods but it seems likely that similar ideas could be exploited in model-based RL. We note that defining a differentiable version of RMs might be necessary to make progress in all these future work directions, since it will allow us to tackle these problems using gradient-based optimization techniques.

Finally, we see many opportunities for using formal languages to define correct reward specification via reward machines and defining novel RL methodologies to exploit the knowledge within reward machines -- resulting in agents that can understand humans' instructions and use them to solve problems faster. To keep pushing in this direction, it might be worth exploring the potential benefits of ascending the Chomsky hierarchy and studying combinations of reward machines with context-free and context-sensitive grammars.

\acks{We gratefully acknowledge funding from the Natural Sciences and Engineering Research Council of Canada (NSERC), the Canada CIFAR AI Chairs Program, and Microsoft Research. Resources used in preparing this research were provided, in part, by the Province of Ontario, the Government of Canada through CIFAR, and companies sponsoring the Vector Institute for Artificial Intelligence \url{www.vectorinstitute.ai/partners}.} 
We also acknowledge the rich multi-disciplinary research environment at the Schwartz Reisman Institute. 
This work was done while the first author was a Ph.D. student at the University of Toronto.


\bibliographystyle{theapa}
\bibliography{references/related_work}

\end{document}